\documentclass{article}

\PassOptionsToPackage{numbers, compress}{natbib}

\usepackage[main,final]{neurips_2026}

\usepackage[utf8]{inputenc} 
\usepackage[T1]{fontenc}    
\usepackage{hyperref}       
\usepackage{url}            
\usepackage{booktabs}       
\usepackage{amsfonts}       
\usepackage{nicefrac}       
\usepackage{microtype}      
\usepackage{xcolor}         
\usepackage{booktabs}
\usepackage{multirow}
\usepackage{graphicx}
\usepackage{times}
\usepackage{latexsym}

\usepackage[T1]{fontenc}

\usepackage[utf8]{inputenc}

\usepackage{microtype}

\usepackage{inconsolata}
\usepackage{times}
\usepackage{latexsym}
\usepackage{amsmath}
\usepackage{mathtools}
\usepackage[T1]{fontenc}

\usepackage[utf8]{inputenc}
\usepackage{tcolorbox}
\usepackage{microtype}

\usepackage{inconsolata}
\usepackage{pifont}
\usepackage{graphicx}

\usepackage{booktabs}   
\usepackage{multirow}
\usepackage{colortbl}   
\usepackage{xcolor}   
\usepackage{booktabs}
\usepackage{multirow}
\usepackage{array}
\usepackage{amsmath}
\usepackage{yhmath}
\usepackage{amssymb}
\usepackage{amsmath,amssymb} 
\usepackage{dashrule}        
\usepackage{setspace} 
\usepackage{arydshln} 
\usepackage{tikz}
\usepackage{footnotehyper}
\usepackage{graphicx}
\usepackage{wrapfig}
\usepackage{marvosym}

\usepackage{amsthm}
\usepackage{makecell}
\newtheorem{definition}{Definition}
\newtheorem{lemma}{Lemma}

\newtheorem{corollary}{Corollary}
\newtheorem{theorem}{Theorem}

\title{Your Group-Relative Advantage Is Biased}

\author{%
\textbf{Fengkai Yang}\textsuperscript{1,3,4}\quad
\textbf{Zherui Chen}\textsuperscript{2}\quad
\textbf{Xiaohan Wang}\textsuperscript{4}\quad
\textbf{Xiaodong Lu}\textsuperscript{1,4}\quad
\textbf{Jiajun Chai}\textsuperscript{4}
\\
\textbf{Guojun Yin}\textsuperscript{4}\quad
\textbf{Wei Lin}\textsuperscript{4}\quad
\textbf{Shuai Ma}\textsuperscript{1}\quad
\textbf{Fuzhen Zhuang}\textsuperscript{1}\thanks{Fuzhen Zhuang is affiliated with the School of Artificial Intelligence, Beihang University.}\quad
\textbf{Deqing Wang}\textsuperscript{1}
\\
\textbf{Yaodong Yang}\textsuperscript{3}\quad
\textbf{Jianxin Li}\textsuperscript{1}\quad
\textbf{Yikun Ban}\textsuperscript{1}\thanks{Corresponding Author. Yikun Ban is affiliated with the State Key Laboratory of Complex \& Critical Software Environment, Beihang University.  If you have any questions, feel free to contact yikunb@buaa.edu.cn or yangfengkai@stu.pku.edu.cn.}
\\[0.7em]
\textsuperscript{1}Beihang University \quad
\textsuperscript{2}University of California, Berkeley \quad
\textsuperscript{3}Peking University \quad
\textsuperscript{4}Meituan
\\[0.3em]
\texttt{yangfengkai@stu.pku.edu.cn}
\quad
\texttt{yikunban@buaa.edu.cn}
\\[0.2em]
}

\begin{document}

\maketitle

\begin{abstract}
  Reinforcement Learning from Verifier Rewards (RLVR) has emerged as a widely used approach for post-training large language models on reasoning tasks, with group-based methods such as GRPO and its variants gaining broad adoption. These methods rely on group-relative advantage estimation to avoid learned critics, yet its theoretical properties remain poorly understood. In this work, we characterize conditional estimation bias in group-relative
advantages among non-degenerate groups. For mean-centered advantages under
binary rewards, our analysis shows negative conditional mean estimation
errors for hard prompts and positive errors for easy prompts.
At the response level, correct answers to hard prompts are underweighted,
while incorrect answers to easy prompts receive insufficient penalties. To address this issue, we propose History-Aware Adaptive Difficulty Weighting (\textbf{HA-DW}), an adaptive reweighting scheme that adjusts advantage estimates based on an evolving difficulty anchor and training dynamics. Both theoretical analysis and experiments on five mathematical reasoning benchmarks demonstrate that HA-DW consistently improves performance when integrated into GRPO and its variants. Our results suggest that correcting biased advantage estimation is critical for robust and efficient RLVR training. The code is available at \textcolor{blue}{https://github.com/yangfengkai/HA-DW}
\end{abstract}

\section{Introduction}

After the success of DeepSeek-R1 \citep{DBLP:journals/corr/abs-2501-12948}, \textit{Reinforcement Learning with Verifiable Rewards} (RLVR) has rapidly emerged as a simple yet powerful paradigm for training reasoning-oriented Large Language Models (LLMs). The key driver behind this leap forward is \textit{Group Relative Policy Optimization} (GRPO) \citep{DBLP:journals/corr/abs-2402-03300}. GRPO leverages \emph{within-group} average reward as a baseline estimator to compute advantages with a small number of generated rollouts for each sampled prompt, thereby avoiding the need for an additional critic model. The group-relative policy optimization paradigm it introduces has gained significant attention due to its simplicity and practical effectiveness. Numerous variants have been proposed to boost training stability and performance, expanding applicability to broader domains such as agent and multimodal training. Representative approaches include GSPO \citep{DBLP:journals/corr/abs-2507-18071}, DAPO \citep{DBLP:journals/corr/abs-2503-14476}, Dr.GRPO \citep{DBLP:journals/corr/abs-2503-20783}, GiGPO \citep{feng2025group}, and StepGRPO \citep{zhang2025r1}. 
However, despite its empirical success, the \emph{intra-group} advantage estimation still lacks detailed theoretical characterization which is critical to the performance in group-relative paradigm \citep{DBLP:journals/corr/abs-2510-04996, DBLP:journals/corr/abs-2509-25300}.

In this paper, we characterize the conditional mean estimation error of
group-relative advantages within non-degenerate groups.
For mean-centered advantages under binary rewards, this error is negative
for \emph{hard prompts} and positive for \emph{easy prompts}.
A separate reward-conditioned analysis shows that correct responses to
hard prompts receive underestimated positive advantages, while incorrect
responses to easy prompts receive insufficiently negative advantages,
as presented in Section~\ref{sec:theoretical analysis2.2}. Such systematic bias can cause the policy to under-learn
from hard questions while over-exploiting easy ones, ultimately hurting both
training stability and generalization.
As illustrated by the representative example in
Figure~\ref{fig:introduction}(b), group-relative estimation can introduce
substantial bias in advantage estimation for group-based RL algorithms.
Our empirical results further corroborate this phenomenon, with consistent
evidence reported in Appendix~\ref{appendix:toy-experiment}.

Motivated by these findings, we propose a novel policy optimization algorithm \emph{History-Aware Adaptive Difficulty Weighting (HA-DW)} that leverages an evolving difficulty anchor that integrates long-term reward trends and historical training information and adaptively reweights advantage estimates based on the group-relative advantage estimator to mitigate the bias induced by group-based advantage estimation. As a plug-and-play module, it amplifies the signal of hard prompts while reducing overexploitation of easy ones and enables a more
principled balance between exploration and exploitation in RL training.
This emphasis on learning near the model's capability boundary aligns with the epistemic exploration perspective, which highlights exploration as a mechanism for sustained capability improvement~\citep{ban2026epistemic}.
The overall framework is depicted in
Figure~\ref{fig:main-picture}. Our goal is not to model all RLVR settings, but to expose a previously overlooked statistical bias in group-relative algorithms and demonstrate that even lightweight corrections can yield consistent gains.
Our main contributions are summarized as follows:

\noindent \textbf{[Discovery]}.
We characterize difficulty-dependent conditional mean estimation errors
in group-relative advantages and provide reward-conditioned results
showing underweighted correct responses to hard prompts and
insufficiently penalized incorrect responses to easy prompts.
    
\noindent  \textbf{[Algorithm]}.
  We propose HA-DW that dynamically
adjusts advantage weights, theoretically effectively compensate
for the bias of group-relative advantage estimation.

\noindent \textbf{[Performance]}.
    As illustrated in Figure~\ref{fig:introduction}(a), we validate our approach through extensive experiments on mathematical and code generation benchmarks with robust improvements. Notably, it achieves performance equivalent to increasing sampling under a limited computational budget.

\section{Why Your Advantage Estimation is Biased?}\label{sec:theoretical-analysis}

\begin{figure*}[t]
    \centering
    \includegraphics[width= \linewidth, trim=0 135 0 95,clip]{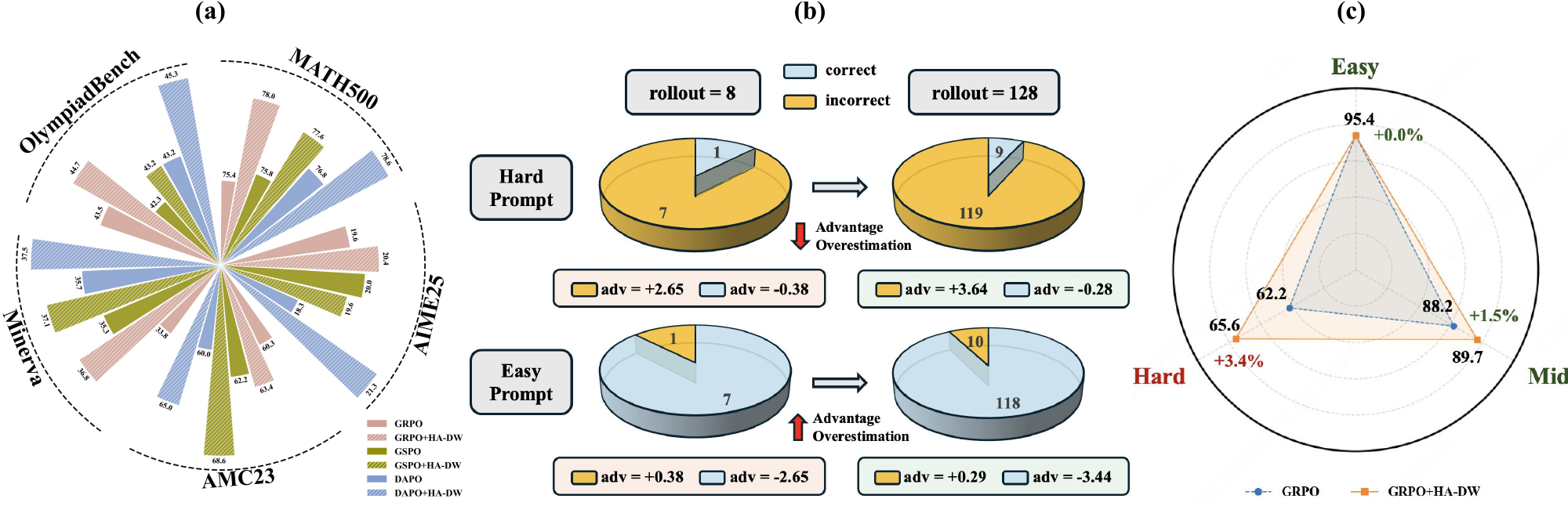}
    \caption{(a) Comparison of the performance of RL algorithms with and without HA-DW on Qwen3-4B-Base across five mathematical reasoning benchmarks. (b) Significant biased advantage estimation on the MATH dataset under 8 and 128 rollouts. (c) Performance gain by GRPO+HA-DW on MATH500 stratified by difficulty levels.
    }
    \vspace{-0.5cm}  
    \label{fig:introduction}
\end{figure*}

In this section, we theoretically analyze the biased estimation in group-relative algorithms. For clarity of exposition, the derivation in this section mainly focuses on the mean-centered group-relative advantage without standard-deviation normalization.  Firstly, we provide the prerequisite definitions.

\subsection{Definitions}

At training step $t$, we sample a prompt $x_t \sim D$. Given $x_t$, a group-relative RL algorithm samples $G$ responses
$\{y_{t,i}\}_{i=1}^G$ independently from the current policy $\pi_{\theta_t}(\cdot \mid x_t)$.
Each response $y_{t,i}$ receives a corresponding scalar reward $r_{t,i} \in \{0,1\}$, forming the reward set $\{r_{t,i}\}_{i=1}^G$, where $r(\cdot)$ is the reward function and we denote $r(y_{t,i})$ by $r_{t,i}$ for brevity.
The \textbf{group-relative advantage} $\hat{A}_{t,i}$ is computed as:
\begin{equation}
\label{eq:group_advantage}
\hat{A}_{t,i}
=
r_{t,i}
-
\hat{p}_t,  \quad \ \hat{p}_t
=
\frac{1}{G}
\sum_{i=1}^G r_{t,i},  
\end{equation} 
where  $\hat{p}_t$ is the group baseline.
Let $\pi_{\theta_{\text{old}}}$ denote the reference (behavior) policy. Then, the group-relative policy optimization objective is defined as:
\begin{equation}
\label{eq:group_po_objective}
J_{\text{group}}(\theta)
=
\frac{1}{G}
\sum_{i=1}^G
\psi\!\left(
\frac{\pi_{\theta}(y_{t,i} \mid x_t)}
     {\pi_{\theta_{\text{old}}}(y_{t,i} \mid x_t)}
\right)
\,
\phi\!\left(\hat{A}_{t,i}\right),
\end{equation}
Here, $\psi(\cdot)$ denotes a function applied to the importance sampling ratio (e.g., identity, clipping, or logarithmic transformation), and $\phi(\cdot)$ is a transformation applied to the advantage term. These two functions are introduced to unify a broad family of group-relative policy optimization variants under a single formulation.

\begin{definition}[Expected Reward]
Assume the reward function $r(\cdot)$ is binary, i.e., $r(\cdot) \in \{0,1\}$.
Given a prompt $x_t \sim D$ and a policy $\pi_{\theta_t}$, let
$y_t \sim \pi_{\theta_t}(\cdot \mid x_t)$ be a sampled response.
The expected reward of policy $\pi_{\theta_t}$ on prompt $x_t$ is defined as:
\begin{equation}
p_t
=
\mathbb{E}_{y_t \sim \pi_{\theta_t}(\cdot \mid x_t)}
\bigl[ r(y_t) \bigr]
=
\mathbb{P}\!\left(r(y_t) = 1 \mid x_t, \pi_{\theta_t}\right).
\end{equation}
\end{definition}

In the RLVR setting, $p_t$ represents the expected reward under policy $\pi_{\theta_t}$ given $x_t$, while $\hat{p}_t$ can be regarded as an empirical estimator of $p_t$ obtained from a finite group of sampled responses. This motivates the following definition.

\begin{definition}[Expected Advantage]
\label{def:expected_advantage}
Given a prompt $x_t \sim D$, let $y_{t,i} \sim \pi_{\theta_t}(\cdot \mid x_t)$ be a sampled
response with corresponding reward $r_{t,i}$.
The expected advantage is defined as:
\begin{equation}
\label{eq:expected_advantage}
A_{t,i}
=
r_{t,i}
-
p_t.
\end{equation}
\end{definition}

Thus, in the RLVR setting, $A_{t,i}$ represents the \emph{expected} advantage of response $y_{t,i}$ under policy $\pi_{\theta_t}$ given $x_t$, while $\hat{A}_{t,i}$ can be regarded as an empirical estimator of $A_{t,i}$ obtained from a finite group of sampled responses. Most group-relative RL algorithms rely on $\hat{A}_{t,i}$ for policy
updates, differing primarily in how $\hat{A}_{t,i}$ is processed or transformed within their respective optimization objectives.

\subsection{Fundamental Discovery}\label{sec:theoretical analysis2.2}

Next, we present a formal formulation of the problem.
Given a prompt $x_t \sim D$, let $p_t$ denote the expected reward of policy
$\pi_{\theta_t}$ on $x_t$. We then sample $G$ responses independently from
$\pi_{\theta_t}(\cdot \mid x_t)$. 
In RLVR, rewards are often
binary, especially in mathematical and formal reasoning tasks where verifiers
return pass/fail signals. Under this widely adopted setting, it is natural to
model the reward associated with each response as a Bernoulli random variable:
\begin{equation}
r_{t,i} \sim \mathrm{Bernoulli}(p_t), \quad \forall i \in [G].
\end{equation}

Let $R = \sum_{i=1}^G r_{t,i}$ denote the total reward within the group.
The empirical group baseline is given by $\hat{p}_t = R / G$.

\begin{definition}[Prompt Difficulty]
\label{def:prompt_difficulty}
Given a prompt $x_t$, a policy $\pi_{\theta_t}$, and a
classification threshold $\tau_{\mathrm{cls}}\in[0,1/2)$, we define:
\begin{itemize}
    \item $x_t$ is a \textbf{hard prompt} if
    $p_t<1/2-\tau_{\mathrm{cls}}$;
    \item $x_t$ is a \textbf{moderate prompt} if
    $1/2-\tau_{\mathrm{cls}}\le p_t\le1/2+\tau_{\mathrm{cls}}$;
    \item $x_t$ is an \textbf{easy prompt} if
    $p_t>1/2+\tau_{\mathrm{cls}}$.
\end{itemize}
The actual difficulty gap is $d_t:=|p_t-1/2|$.
The threshold $\tau_{\mathrm{cls}}$ determines class membership,
whereas $d_t$ measures the prompt's distance from $p_t=1/2$.
Unless otherwise stated, the subsequent analysis uses
$\tau_{\mathrm{cls}}=0$.
\end{definition}

In group-based policy optimization, the group-relative advantage estimator
satisfies $\hat{A}_{t,i}=0$ for all $i\in[G]$ when either $R=0$ or $R=G$,
resulting in zero gradients and hence no parameter updates.
In practice, such degenerate groups do not contribute to learning and are
either explicitly discarded or implicitly ignored by GRPO-style algorithms. Accordingly, our analysis focuses on the effective update regime, namely
groups for which at least one response receives a non-zero advantage.
This corresponds to the non-degenerate event:
\begin{equation}
\mathcal{S} \coloneqq \{1 \le R \le G-1\}.
\end{equation}
Conditioning on $\mathcal{S}$ isolates non-degenerate groups for the
analysis of advantage estimation error. The resulting statements concern
this conditional distribution; they do not assert that the
unconditionally sampled group baseline is biased.
Next, we present the main results.

\begin{tcolorbox}[
    colback=gray!5,
    colframe=black,
    boxrule=0.8pt,
    arc=2pt,
    left=6pt,
    right=6pt,
    top=6pt,
    bottom=6pt
]
\begin{theorem}[Conditional Mean Advantage Error]
\label{theo:main1}
Fix a prompt $x_t$ and the current policy $\pi_{\theta_t}$.
Suppose $G\ge2$ and
$r_{t,1},\ldots,r_{t,G}\overset{\mathrm{i.i.d.}}{\sim}
\mathrm{Bernoulli}(p_t)$ with $p_t\in(0,1)$.
Condition on $\mathcal{S}=\{1\le R\le G-1\}$.
For every $i\in[G]$, the conditional mean estimation error satisfies
\[
\begin{cases}
\mathbb{E}[\hat{A}_{t,i}-A_{t,i}\mid\mathcal{S}]<0, & p_t<1/2,\\
\mathbb{E}[\hat{A}_{t,i}-A_{t,i}\mid\mathcal{S}]=0, & p_t=1/2,\\
\mathbb{E}[\hat{A}_{t,i}-A_{t,i}\mid\mathcal{S}]>0, & p_t>1/2.
\end{cases}
\]
\end{theorem}
\end{tcolorbox}

Theorem~\ref{theo:main1} characterizes the conditional mean estimation
error within non-degenerate groups; its closed-form expression and proof
are provided in Appendix~\ref{proof of theorem1}. Unconditionally,
$\mathbb{E}[\hat{p}_t]=p_t$ and
$\mathbb{E}[\hat{A}_{t,i}-A_{t,i}]=0$.
Conditioning on $\mathcal{S}$ shifts the expected group baseline upward
when $p_t<1/2$ and downward when $p_t>1/2$, producing negative and
positive conditional mean errors, respectively.
The zero conditional mean of the centered advantage does not imply
a zero conditional mean estimation error.
The latter vanishes only at $p_t=1/2$; this does not imply that the
error also vanishes after conditioning on the response reward.

\begin{corollary}[Reward-Conditioned Advantages]
\label{cor:reward_conditional}
Under the assumptions of Theorem~\ref{theo:main1},
if $p_t<1/2$, then
\[
\mathbb{E}[\hat{A}_{t,i}\mid r_{t,i}=1,\mathcal{S}]<1-p_t.
\]
If $p_t>1/2$, then
\[
\mathbb{E}[\hat{A}_{t,i}\mid r_{t,i}=0,\mathcal{S}]>-p_t.
\]
\end{corollary}

Thus, in conditional expectation, a correct response to a hard prompt
receives an underestimated positive advantage, while an incorrect
response to an easy prompt receives an insufficiently negative advantage.
The comparison targets $1-p_t$ and $-p_t$ are fixed once the response
reward is specified.
The closed-form expressions and proof are provided in Appendix~\ref{proof:reward_conditional}.

\paragraph{Dependence on the difficulty gap.}
For $G\ge2$ and $d\in[0,1/2)$, define
\begin{equation}
\label{eq:difficulty_gap_bias}
B_G(d):=
\frac{
(\frac12-d)(\frac12+d)^G-(\frac12+d)(\frac12-d)^G
}{
1-(\frac12+d)^G-(\frac12-d)^G
}.
\end{equation}
By symmetry, the absolute conditional mean error satisfies
\[
\left|\mathbb{E}[\hat{A}_{t,i}-A_{t,i}\mid\mathcal{S}]\right|
=B_G(d_t),\qquad B_G(0)=0.
\]
For each fixed $G\ge2$, as $d\to0$,
\begin{equation}
\label{eq:difficulty_gap_expansion}
B_G(d)=\frac{G-1}{2^{G-1}-1}\,d+O_G(d^3).
\end{equation}
The leading coefficient is positive and decreases with $G$.
Hence the conditional mean bias grows locally with the actual
distance from $p_t=1/2$, while increasing the group size reduces
this local sensitivity.
Setting $\tau_{\mathrm{cls}}=0$ changes the classification boundary;
the conditional mean error vanishes when $d_t=0$, equivalently
$p_t=1/2$.
These statements concern the mean-centered error conditioned on
$\mathcal{S}$, before standard-deviation normalization.
The derivation is provided in Appendix~\ref{proof:difficulty_gap}.

However, the expectation-level result in Theorem~\ref{theo:main1} alone is insufficient to characterize the \emph{probability} of overestimation or underestimation of $\hat{A}_{t,i}$.
We provide the following probabilistic result.

\begin{tcolorbox}[
    colback=gray!5,
    colframe=black,
    boxrule=0.8pt,
    arc=2pt,
    left=6pt,
    right=6pt,
    top=6pt,
    bottom=6pt
]
\begin{theorem}\label{theo:hard regime lower bound}
Under the condition of Theorem \ref{theo:main1}, suppose $x_t$ is a \textbf{hard} prompt ($p_t < 0.5$). Then, for any $\epsilon\in \left(0, \mathbb{E}[\hat{p}_t\mid \mathcal{S}]-p_t\right)$, we have:
\[
\begin{aligned}
    &\mathbb{P}\left(A_{t,i} - \hat{A}_{t,i} > \epsilon \mid \mathcal{S} \right) 
    = \frac{\sum_{k=\lfloor G\left(p_t+\epsilon\right) \rfloor + 1}^{G-1} \binom{G}{k} \, p_t^{k} (1-p_t)^{G-k}}{1 - (1-p_t)^G - p_t^G} .
\end{aligned}
\]
Simlarily, suppose $x_t$ is an \textbf{easy} prompt ($p_t > 0.5$). Then, for any $\epsilon\in \left(0, p_t - \mathbb{E}[\hat{p}_t\mid \mathcal{S}] \right)$, we have:
\[
\begin{aligned}
    &\mathbb{P}\left(\hat{A}_{t,i} - A_{t,i} > \epsilon \mid \mathcal{S}\right) 
    = \frac{\sum_{k=1}^{\lceil G\left(p_t - \epsilon\right) \rceil - 1} \binom{G}{k} \, p_t^{k} (1-p_t)^{G-k}}{1 - (1-p_t)^G - p_t^G} .
\end{aligned}
\]
\end{theorem}
\end{tcolorbox}

Theorem~\ref{theo:hard regime lower bound} provides a distribution-level
characterization of how likely group-relative advantage estimation is to
\emph{underestimate} or \emph{overestimate} the true advantage, depending on
prompt difficulty.
In contrast to expectation-level results, this theorem quantifies the exact
probability mass of large estimation errors under finite group sizes.
The conditional concentration bound and mean-bias rates are provided in
Appendix~\ref{proof:conditional_error_rates}.

It is well known that generating multiple rollouts per prompt is computationally
expensive in practice. Consequently, existing RLVR methods typically sample only
a small number of responses (e.g., $G=8$) for each prompt $x_t$ to
estimate $\hat{p}_t$ \citep{DBLP:journals/corr/abs-2505-08364, liu2025spec_rl, DBLP:journals/corr/abs-2508-20151} .
Motivated by this practical constraint, we derive the following corollaries
based on Theorem~\ref{theo:hard regime lower bound}, which explicitly characterize
the estimation behavior under small group sizes.

\begin{corollary}
\label{cor:prob_bias}
Under the condition of Theorem~\ref{theo:hard regime lower bound}, suppose the
group size satisfies $2 \le G \le 8$, and assume that $p_t$ is uniformly
distributed over $[0,1]$.
Then, for any $i \in [G]$, the following inequalities hold:
\[ 
\begin{array}{l@{\qquad}l}
\mathbb{P}\!\left(\hat{A}_{t,i} < A_{t,i} \mid \mathcal{S},\; p_t < 0.5 \right)
> 0.63,
&
\mathbb{P}\!\left(\hat{A}_{t,i} > A_{t,i} \mid \mathcal{S},\; p_t > 0.5 \right)
> 0.63,
\\[0.4em]

\mathbb{P}\!\left(\hat{A}_{t,i} < A_{t,i} \mid \mathcal{S},\; p_t < 0.25 \right)
> 0.78,
&
\mathbb{P}\!\left(\hat{A}_{t,i} > A_{t,i} \mid \mathcal{S},\; p_t > 0.75 \right)
> 0.78,
\\[0.4em]

\mathbb{P}\!\left(\hat{A}_{t,i} < A_{t,i} \mid \mathcal{S},\; p_t < 0.125 \right)
= 1.00,
&
\mathbb{P}\!\left(\hat{A}_{t,i} > A_{t,i} \mid \mathcal{S},\; p_t > 0.875 \right)
= 1.00.
\end{array}
\]
\end{corollary}

Corollary~\ref{cor:prob_bias} shows that, with high probability, the group-relative
advantage estimator $\hat{A}_{t,i}$ \emph{underestimates} the true advantage
$A_{t,i}$ for hard prompts and \emph{overestimates} $A_{t,i}$ for easy prompts,
 under the practical set of $G$.
These probability statements concern increasingly extreme ranges of
$p_t$. The dependence of the conditional mean bias on the actual
difficulty gap $d_t$ is characterized by
Equations~\eqref{eq:difficulty_gap_bias}
and~\eqref{eq:difficulty_gap_expansion}.

\begin{corollary} \label{corollary:sufficient G}
Under the condition of Corollary \ref{cor:prob_bias},
suppose $G\geq 6$. The following inequalities hold:
\begin{equation}
    \mathbb{P}\!\left(\hat{A}_{t,i} < A_{t,i} \mid \mathcal{S},\; p_t < \frac{2}{G} \right)
> 0.78, \quad \mathbb{P}\!\left(\hat{A}_{t,i} > A_{t,i} \mid \mathcal{S},\; p_t > \frac{G-2}{G} \right)
> 0.78.
\end{equation}
\end{corollary}

\begin{corollary}
\label{cor:extreme_sure}
Under the condition of Theorem \ref{theo:hard regime lower bound},
suppose $G \ge 2$.
Then, for any $i \in [G]$, the following inequalities hold surely:
\begin{equation}
\hat{A}_{t,i} < A_{t,i} \text{ as } p_t < \tfrac{1}{G}, \qquad
\hat{A}_{t,i} > A_{t,i} \text{ as } p_t > \tfrac{G-1}{G} .
\end{equation}
\end{corollary}
Corollary~\ref{cor:extreme_sure} shows that the group-based advantage estimator
$\hat{A}_{t,i}$ necessarily \emph{underestimates} the true advantage $A_{t,i}$
for extremely difficult prompts ($p_t < 1/G$), and \emph{overestimates} $A_{t,i}$
for extremely easy prompts ($p_t > (G-1)/G$).
Detailed derivation process is presented in Appendix~\ref{appendix:theoretical proof}.


\subsection{Auxiliary Results}

\begin{theorem}[Conditional Bias under Bounded Continuous Rewards]
\label{thm:nonbinary_bias}
At training step $t$, let $G\ge 2$ and draw i.i.d. bounded rewards
$r_{t,1},\ldots,r_{t,G}\sim \mathcal{D}(p_t)$ with $r_{t,i}\in[0,1]$ and
$\mathbb{E}[r_{t,i}]=p_t$. For a threshold $\sigma\in[0,1]$, consider the update event $S_\sigma=\left\{\exists\, i\neq j:\ |r_{t,i}-r_{t,j}|>\sigma\right\}$. Assume $\mathbb{P}(S_\sigma)>0$. Then, conditioned on performing an update, the advantage estimation bias satisfies:
\begin{equation}
    \mathbb{E}\!\left[\hat A_{t,i}-A_{t,i}\mid S_\sigma\right]
    =
    p_t-\mathbb{E}\!\left[\hat p_t\mid S_\sigma\right],
    \qquad \forall i\in\{1,\ldots,G\}.
\end{equation}
Consequently, whenever the update event $S_\sigma$ changes the conditional expectation of the group mean, i.e.,
\begin{equation}
    \mathbb{E}\!\left[\hat A_{t,i}-A_{t,i}\mid S_\sigma\right]\neq0,
\end{equation}
the group-relative advantage estimator is conditionally biased, even under bounded continuous rewards.
Moreover, the bias direction is deterministic in the following extreme regimes:
\[
\begin{cases}
\displaystyle
\mathbb{E}\!\left[\hat A_{t,i}-A_{t,i}\mid S_\sigma\right]
<
0,
& \text{if } p_t\le \frac{\sigma}{G}, \\[2ex]
\displaystyle
\mathbb{E}\!\left[\hat A_{t,i}-A_{t,i}\mid S_\sigma\right]
>
0,
& \text{if } p_t\ge 1-\frac{\sigma}{G}.
\end{cases}
\]
Thus, the conditional bias is not an artifact of binary rewards; it persists
under bounded continuous rewards whenever the update-selection event distorts the conditional group mean.
\end{theorem}

While the binary reward assumption covers many standard RLVR setups, we recognize that real-world reward signals can be more general. To address this, we extend our analysis  to \textbf{continuous bounded reward distributions}. Theorem~\ref{thm:nonbinary_bias} suggests that the core bias phenomenon is not an artifact of the Bernoulli reward assumption but is prevalent across a broader class of bounded reward models. The more detailed analysis is presented in Appendix~\ref{appendix:nonbinary}.

\begin{theorem}[Bias under Standard-Deviation Normalization]
\label{thm:std-normalization-bias}
Under the retention condition 
\(\mathcal{S}=\{1\le R\le G-1\}\), for a correct response
\(r_{t,i}=1\), the standard-deviation-normalized group-relative advantage satisfies:
\begin{equation}
    \mathbb{E}\!\left[
    \frac{\hat A_{t,i}}{\sqrt{\hat p_t(1-\hat p_t)}}
    \, \mid  \,
    r_{t,i}=1,\mathcal{S}
\right]
\neq
\frac{A_{t,i}}{\sqrt{p_t(1-p_t)}} .
\end{equation}

Moreover, the bias direction is deterministic in the extreme regimes:
\[
\begin{cases}
\displaystyle
\mathbb{E}\!\left[
    \frac{\hat A_{t,i}}{\sqrt{\hat p_t(1-\hat p_t)}}
    \mid 
    r_{t,i}=1,\mathcal{S}
\right]
<
\frac{A_{t,i}}{\sqrt{p_t(1-p_t)}},
& \text{if } p_t<\frac{1}{G}, \\[2ex]
\displaystyle
\mathbb{E}\!\left[
    \frac{\hat A_{t,i}}{\sqrt{\hat p_t(1-\hat p_t)}}
    \mid 
    r_{t,i}=1,\mathcal{S}
\right]
>
\frac{A_{t,i}}{\sqrt{p_t(1-p_t)}},
& \text{if } p_t>\frac{G-1}{G}.
\end{cases}
\]
Thus, standard-deviation normalization rescales the group-relative advantage
but does not eliminate its intrinsic conditioning bias.
\end{theorem}

We further examine whether the bias can be removed by the standard-deviation
normalization commonly used in group-relative policy optimization. Since this
normalization divides the empirical advantage by the within-group reward
standard deviation, one might expect it to correct the scale distortion caused
by varying prompt difficulty. However, the normalization is itself computed from
the same retained group samples and is therefore affected by the conditioning
event. As a result, the estimator remains biased even after normalization. Normalization mainly stabilizes the scale of policy updates, whereas HA-DW targets the difficulty-dependent distortion induced by conditional group selection. The
above theorem shows that this bias persists for correct responses and that
its direction is deterministic in the extreme difficulty regimes. And the detailed derivation with standard-deviation normalization is provided in Appendix~\ref{appendix:std_norm}.

\section{Proposed Solution}

Since the group-based advantage estimator is biased, we propose an algorithm to adjust the advantage estimation accordingly. The proposed approach consists of two key components: \textbf{Evolving Difficulty Anchor} and \textbf{History Aware Adaptive Difficulty Weighting}.


\subsection{Evolving Difficulty Anchor}

To track the evolving model state across batches, we propose the cross-batch difficulty anchor
framework that integrates long-term reward trends and historical information.
Let $B_t$ denote the total number of responses in batch $t$.
Model updates are guided by observations of the current batch’s prompt accuracy
$b_t$ together with historical information, defined as:
\begin{equation}
    b_t = \frac{K_t}{B_t}, \qquad
    K_t = \sum_{i=1}^{B_t} r_{t,i}.
\end{equation}

We treat the model’s solving capability $C_t$ as a latent belief state.
At training step $t$, the observation $b_t$ is used to update the prior belief
$C_t^{-}$ to the posterior belief $C_t^{+}$ via a Kalman-style update
\citep{DBLP:journals/automatica/BattilottiBCDG26, DBLP:journals/automatica/Zhang26}:
\begin{equation}\label{eq:soft-evolve}
    C_t^{+} = (1 - \sigma_t)\, C_t^{-} + \sigma_t\, b_t,
    \quad \eta_t \in [0,1].
\end{equation}

The forgetting factor $\sigma_t$ controls the influence of historical information
and is dynamically modulated by model stability. Specifically, we compute the
average belief over the previous $m$ batches as:
\begin{equation}
    \bar{C}_{t} = \frac{1}{m} \sum_{j=1}^{m} C_{t-j},
\end{equation}
and define the $\sigma_t$ as corresponding standard deviation: $\sigma_t =
    \sqrt{
        \frac{1}{m}
        \sum_{j=1}^{m}
        \left(C_{t-j} - \bar{C}_t\right)^2
    }$. Between consecutive steps, the posterior belief $C_t^{+}$ serves as the prior
belief for the next batch: $C_t^{+} \rightarrow C_{t+1}^{-}$. Overall, $C_t$ enables the model to aggregate information across historical
batches via belief updates and to condition its training strategy on this
evolving belief.

\subsection{History Aware Adaptive Difficulty Weighting \textbf{(HA-DW)}}

\begin{figure*}[t]
    \centering
    \includegraphics[width=\linewidth, trim=0 110 0 80,clip]{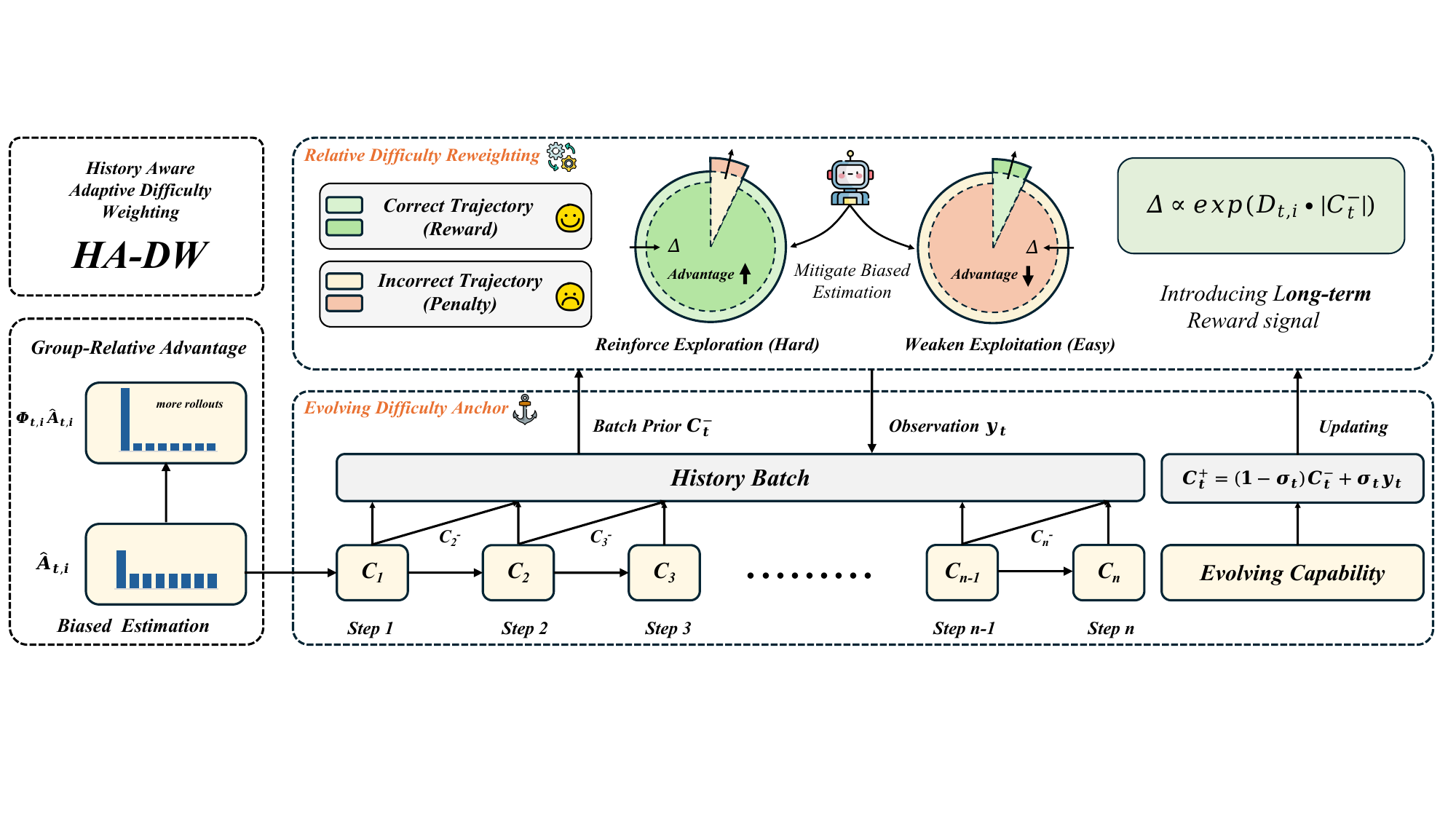}
    \caption{HA-DW consists of two collaborative phases. In the first phase, an evolving difficulty anchor incorporates cross-batch historical information by propagating the model’s prior through a history buffer, capturing long-term reward trends. In the second phase, prompt weights are adaptively adjusted based on their estimated difficulty under the model’s evolving state, compensating for biased advantage estimates.}
    \label{fig:main-picture}
\end{figure*}

To rectify the inherent bias in group-based advantage estimation, we introduce
HA-DW, which dynamically
adjusts advantage weights based on the model’s evolving state while
incorporating long-term reward signals.
Coupled with the evolving difficulty anchor, we define the history-based prompt difficulty
as:
\begin{equation}
    \mathrm{diff}^{\mathrm{his}}_t = \hat{p}_t - C_t^{-},
\end{equation}
where $\mathrm{diff}^{\mathrm{his}}_t$ captures both the magnitude and direction
of a prompt’s difficulty relative to the current model belief. To determine the \emph{direction} of adjustment, we use the evolving difficulty anchor as a reference and define:
\begin{equation}
    D_{t,i}
    =
    -\,\mathrm{sgn}\!\left(\hat{A}_{t,i}\right)
    \cdot
    \mathrm{sgn}\!\left(\mathrm{diff}^{\mathrm{his}}_t\right),
\end{equation}
where $\mathrm{sgn}(\cdot)$ denotes the sign function. Next, we quantify the \emph{magnitude} of adjustment using the absolute
history-based difficulty: $M_t = \left|\mathrm{diff}^{\mathrm{his}}_t \right|$. Here, $M_t$ measures the extent to which the prompt deviates from the model’s
current capability. We then define the history-aware reweighting factor as:
\begin{equation}
\label{eq:reweighting factor}
    \Phi_{t,i}
    =
    \lambda_{\mathrm{scale}}
    \cdot
    \exp\!\left(D_{t,i} \cdot M_t\right),
\end{equation}
where $\lambda_{\mathrm{scale}}$ is a scaling constant, and the exponential form
ensures smooth and multiplicative adjustment of advantage weights.
The resulting HA-DW objective is:
\begin{equation}
\begin{aligned}
    L_{\mathrm{HA\text{-}DW}}(\theta)
    =
    \frac{1}{G}
    \sum_{i=1}^{G}
    &\psi\!\left(
        \frac{\pi_{\theta}(y_{t,i} \mid x_t)}
             {\pi_{\theta_{\mathrm{old}}}(y_{t,i} \mid x_t)}
    \right) \cdot \phi\!\left(\hat{A}_{t,i}\right)
    \cdot
    \Phi_{t,i},
\end{aligned}
\end{equation}
where $\psi(\cdot)$ and $\phi(\cdot)$ follow specific definitions in group-relative RL algorithms.

Intuitively, $\Phi_{t,i}$ adjusts  the estimated advantage for each prompt, thereby correcting systematic bias identified in our analysis. HA-DW mitigates distortion within retained non-degenerate groups, rather than recovering gradients from degenerate groups. And it can be seamlessly integrated as a plug-and-play module into group-relative methods, improving reasoning performance under fixed rollouts while effectively mitigating biased advantage estimation. Detailed instantiations for GRPO and related algorithms are provided in the
Appendix~\ref{appendix:loss func}

\section{Theoretical Analysis}

In this section, we provide a theoretical analysis of the effectiveness of the proposed adjustment strategy. We begin by analyzing how reweighting the empirical baseline $\hat{p}_t$ affects the expected bias.

\begin{lemma}[Baseline Rectification]
\label{lemma:adjustment of baseline}
Given a prompt $x_t \sim D$ and the policy $\pi_{\theta_t}$, let
$\tilde{p}_t = c \cdot \hat{p}_t$ be the rectified group baseline.
Assume $p_t \in [\eta_{\mathrm{margin}},\, 1-\eta_{\mathrm{margin}}]$ for some $\eta_{\mathrm{margin}} \in (0,1/2]$.
Given any $\delta \in (0,1)$,
we can define that:
\begin{equation}
\label{eq:eps_delta}
\epsilon_\delta
:=
\sqrt{\frac{1}{2G}\log\!\left(
\frac{2}{\delta\big(1-(1-\eta_{\mathrm{margin}})^G-\eta_{\mathrm{margin}}^G\big)}
\right)}.
\end{equation}
Let
\[
\begin{aligned}
&I_t := \bigl[\hat p_t-\epsilon_\delta,\ \hat p_t+\epsilon_\delta\bigr]\cap[\eta_{\mathrm{margin}},1-\eta_{\mathrm{margin}}], \quad
 A(p_t):=1-(1-p_t)^G-p_t^G.
\end{aligned}
\]
Fix any $\epsilon>0$,
we define:
\begin{equation}
\label{eq:cl_ch_def}
\begin{aligned}
c_{\mathrm{low}}
&:=
\sup_{p_t\in I_t}
\frac{(p_t-\epsilon)\,A(p_t)}{p_t(1-p_t^{G-1})},
\qquad
c_{\mathrm{high}}
:=
\inf_{p_t\in I_t}
\frac{(p_t+\epsilon)\,A(p_t)}{p_t(1-p_t^{G-1})}.
\end{aligned}
\end{equation}
Then, with
probability at least $1-\delta$ conditional on $\mathcal S$, for any choice $c \in (c_{\mathrm{low}},\ c_{\mathrm{high}})$, we can derive that:
\begin{equation}
    \mathbb E[\tilde p_t\mid \mathcal S] \in (p_t-\epsilon,\ p_t+\epsilon).
\end{equation}
\end{lemma}

Specifically, we consider adjusting the empirical group baseline using a
reweighting factor $c$. From the perspective of the expected estimation bias, Lemma~\ref{lemma:adjustment of baseline} that an appropriate choice of $c$ can effectively reduce estimation bias. Detailed derivations are provided in
Appendix~\ref{appendix:adjusted estimator}.
Next, we now proceed to present the main theoretical result.

\begin{tcolorbox}[
    colback=gray!5,
    colframe=black,
    boxrule=0.8pt,
    arc=2pt,
    left=6pt,
    right=6pt,
    top=6pt,
    bottom=6pt
]
\begin{theorem}
\label{theo:algorithm}
Under the condition of Lemma \ref{lemma:adjustment of baseline},
suppose there exists a scaling factor $\lambda_{\mathrm{scale}}$ in
Equation~\eqref{eq:reweighting factor} such that:
\[
\begin{aligned}
\lambda_{\mathrm{scale}}
\in\;&
\left(
\frac{1 + \frac{(1-c_{\mathrm{high}})\hat{p}_t}{1-\hat{p}_t}}
     {\exp(D_{t,i} M_t)},
\;
\frac{1 + \frac{(1-c_{\mathrm{low}})\hat{p}_t}{1-\hat{p}_t}}
     {\exp(D_{t,i} M_t)}
\right)
\cup
\left(
\frac{c_{\mathrm{low}}}{\exp(D_{t,i} M_t)},
\;
\frac{c_{\mathrm{high}}}{\exp(D_{t,i} M_t)}
\right).
\end{aligned}
\]
Then, HA-DW algorithm provably mitigates the bias of group-relative advantage:
\[
\begin{aligned}
    &\left|
\mathbb{E}\!\left[\hat{A}_{t,i}\cdot \Phi_{t,i} \mid \mathcal{S}\right]
- A_{t,i}
\right|<
\left|
\mathbb{E}\!\left[\hat{A}_{t,i} \mid \mathcal{S}\right]
- A_{t,i}
\right|.
\end{aligned}
\]
\end{theorem}
\end{tcolorbox}

Theorem~\ref{theo:algorithm} shows that, with an appropriate choice of the scaling
parameter $\lambda_{\mathrm{scale}}$, the HA-DW adjustment yields advantage
estimates that are closer to the true advantage $A_{t,i}$ in expectation.
This theoretical result provides principled guidance for selecting
$\lambda_{\mathrm{scale}}$ in practice.

\section{Experiments}

\begin{table*}[!t]
\centering
\footnotesize
\setlength{\tabcolsep}{4.2pt}
\renewcommand{\arraystretch}{1.05}

\resizebox{\textwidth}{!}{%
\begin{tabular}{lccccccccc}
\toprule
\multirow{2}{*}{\textbf{Algorithm}} 
& \multicolumn{6}{c}{\textbf{Math}} 
& \multicolumn{3}{c}{\textbf{Code}} \\
\cmidrule(lr){2-7} \cmidrule(lr){8-10}
& \textbf{MATH500} & \textbf{AIME25} & \textbf{AMC23} & \textbf{Minerva} 
& \textbf{OlympiadBench} & \textbf{Avg.}
& \textbf{HumanEval} & \textbf{MBPP} & \textbf{Avg.} \\
\midrule

\multicolumn{10}{c}{\textit{\textbf{Qwen-3-4B-Base}}} \\
\midrule
GRPO
& 75.4 & 19.6 & 60.3 & 33.8 & 43.5 & 46.5 & 70.1 & 58.2 & 64.2 \\
\rowcolor{blue!5}
$\hookrightarrow$ + HA-DW
& 78.0 & 20.4 & 63.4 & 36.8 & 44.7 & \textbf{48.7} & 73.2 & 60.4 & \textbf{66.8} \\
GSPO
& 75.8 & 20.0 & 62.2 & 35.3 & 42.3 & 47.1 & 69.5 & 58.6 & 64.1 \\
\rowcolor{blue!5}
$\hookrightarrow$ + HA-DW
& 77.6 & 19.6 & 68.6 & 37.1 & 43.2 & \textbf{49.2} & 72.0 & 59.4 & \textbf{65.7} \\
DAPO
& 76.8 & 18.3 & 60.0 & 35.7 & 43.2 & 46.8 & 71.3 & 60.0 & 65.6 \\
\rowcolor{blue!5}
$\hookrightarrow$ + HA-DW
& 78.6 & 21.3 & 65.0 & 37.5 & 45.3 & \textbf{49.5} & 73.8 & 62.6 & \textbf{68.2} \\
    
\midrule
\multicolumn{10}{c}{\textit{\textbf{Qwen-3-8B-Base}}} \\
\midrule
GRPO
& 78.8 & 20.4 & 64.2 & 38.2 & 46.4 & 49.6 & 78.7 & 65.6 & 72.2 \\
\rowcolor{blue!5}
$\hookrightarrow$ + HA-DW
& 80.0 & 22.9 & 72.8 & 39.7 & 47.1 & \textbf{52.5} & 81.1 & 68.0 & \textbf{74.6} \\
GSPO
& 78.6 & 21.7 & 67.0 & 37.9 & 45.9 & 50.2 & 78.0 & 64.8 & 71.4 \\
\rowcolor{blue!5}
$\hookrightarrow$ + HA-DW
& 80.2 & 22.1 & 66.5 & 41.9 & 47.6 & \textbf{51.7} & 80.5 & 67.4 & \textbf{74.0} \\
DAPO
& 79.2 & 20.4 & 67.5 & 39.3 & 47.2 & 50.7 & 79.9 & 66.0 & 72.9 \\
\rowcolor{blue!5}
$\hookrightarrow$ + HA-DW
& 82.8 & 23.3 & 70.0 & 40.8 & 50.0 & \textbf{53.4} & 84.1 & 68.4 & \textbf{76.3} \\
    
\midrule
\multicolumn{10}{c}{\textit{\textbf{LLaMA-3.2-3B-Instruct}}} \\
\midrule
GRPO
& 51.4 & 2.7 & 31.7 & 22.8 & 19.9 & 25.7 & 48.8 & 43.2 & 46.0 \\
\rowcolor{blue!5}
$\hookrightarrow$ + HA-DW
& 53.2 & 3.3 & 35.0 & 23.9 & 20.1 & \textbf{27.1} & 51.2 & 45.6 & \textbf{48.4} \\
GSPO
& 48.6 & 1.9 & 30.9 & 23.2 & 19.8 & 24.9 & 49.4 & 43.6 & 46.5 \\
\rowcolor{blue!5}
$\hookrightarrow$ + HA-DW
& 50.4 & 2.3 & 32.7 & 22.4 & 21.0 & \textbf{25.8} & 50.6 & 46.0 & \textbf{48.3} \\
DAPO
& 52.4 & 2.5 & 35.0 & 22.4 & 20.2 & 26.5 & 50.0 & 44.6 & 47.3 \\
\rowcolor{blue!5}
$\hookrightarrow$ + HA-DW
& 53.2 & 3.1 & 37.5 & 24.6 & 22.3 & \textbf{28.1} & 53.0 & 47.2 & \textbf{50.1} \\
\bottomrule
\end{tabular}%
}

\caption{\label{main-results}
Overall results across models (Qwen, LLaMA) and different group-relative RL algorithms (GRPO, GSPO, DAPO). We report the performance of different base RL algorithms, and the corresponding accuracy when applied HA-DW for each model scale and family.
}
\vspace{-0.3cm}
\end{table*}

\paragraph{Setups.} We conduct our experiments on Qwen3-4B-Base, Qwen3-8B-Base \cite{qwen3technicalreport} and LLaMA-3.2-3B-Instruct on five common-used mathematical reasoning benchmarks and two code generation benchmarks. We apply our proposed method on top of several representative group-relative reinforcement learning algorithms: GRPO, GSPO, and DAPO. We compare the performance of group-relative algorithms applying HA-DW to original ones, verifying the effectiveness and scalability of our method.  We conduct RL training within the VeRL framework \citep{sheng2024hybridflow} on a single node with 8 $\times$ NVIDIA A100 GPUs. More implementation details in Appendix \ref{appendix:training-hyperparameters}.

\subsection{Main Results}

\begin{table*}[!t]
  \centering
    \small   
  \setlength{\tabcolsep}{10pt}
  \begin{tabular}{cccccccc}
    \toprule
     \textbf{Threshold} & \textbf{MATH500} & \textbf{AIME25} & \textbf{AMC23} & \textbf{Minerva} & \textbf{OlympiadBench} & \textbf{AVG}\\
    \midrule
    Base
    &75.4 &19.6 &60.3 &33.8 &43.5 &46.5\\
    0.4 (fixed)
    &77.0 &18.5 &63.1 &\textbf{37.5} &44.3 &48.1\\
    0.5 (fixed)
    &76.6 &20.0 &62.7 &35.7 &44.0 &47.8\\
    0.6 (fixed)
    &76.8 &\textbf{21.3} &61.1 &36.4 &44.3 &48.0\\
    $C_t$
    &\textbf{78.0} &20.4 &\textbf{63.4} &36.8 &\textbf{44.7} &\textbf{48.7}\\
\bottomrule
  \end{tabular}
  \caption{\label{ablation study}
    Ablation on the effectiveness of dynamic threshold for RL training using Qwen3-4B-Base. $C_t$ denotes the dynamic threshold.
  }
\vspace{-0.5cm}
\end{table*}

Our main results are presented in Table~\ref{main-results}. Notably, group-based RL algorithms ( GRPO, GSPO and DAPO ) equipped with HA-DW outperforms original methods across seven benchmarks. We observed clear and consistent improvements across benchmarks on different models of different scales and family. Overall, the results underscore that HA-DW compensates for advantage estimation bias via dynamic reweighting to fully leverage these overshadowed critical prompts, thereby unlocking the potential performance gains in RL.

\par To validate our method's effectiveness in extending model capabilities, we divided the MATH500 dataset into three difficulty levels: Easy (Level 1), Mid (Levels 2-3), and Hard (Levels 4-5). We evaluated Qwen3-4B-Base trained with GRPO and GRPO+HA-DW, as shown in Figure~\ref{fig:introduction}(c). The performance on Easy and Mid levels was comparable for both methods, but GRPO+HA-DW outperformed GRPO by $3.4\%$ on Hard prompts. This improvement is due to our history-based dynamic reweighting strategy, which enhances exploration on hard prompts while reducing unnecessary exploitation on easy ones. Simultaneously, it substantiates the existence of bias estimation indirectly.

\paragraph{Training Dynamics.}

\begin{figure*}[!t]
    \centering
    \includegraphics[width=\textwidth]{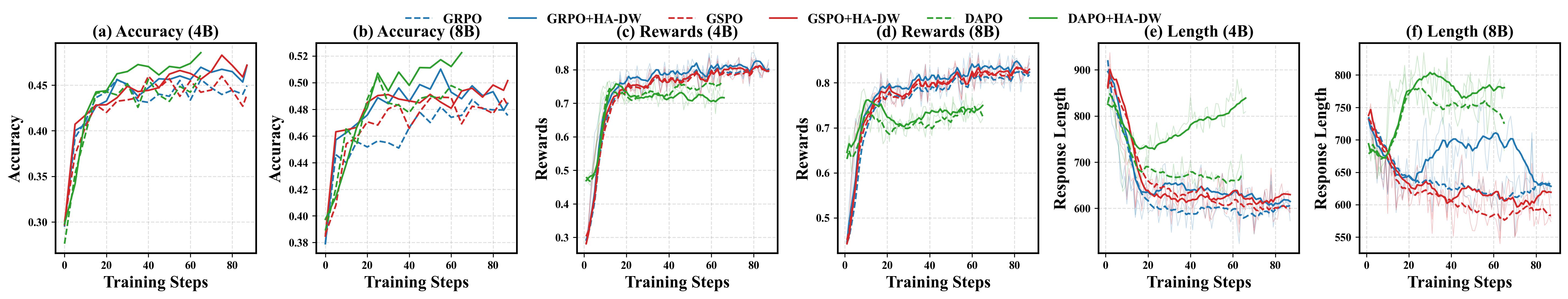}
    \caption{Comparison of training dynamics under different training strategies. Average accuracy across five benchmarks, training reward and response length of Qwen3-4B-Base and Qwen3-8B-Base on different training methods.}
    \label{fig:training-dynamics}
      \vspace{-0.3cm}
\end{figure*}

Figure~\ref{fig:training-dynamics} demonstrates the temporal dynamics of average accuracy across seven benchmarks, training rewards throughout the training process, and response lengths of Qwen3-4B-Base and Qwen3-8B-Base. RL algorithms applied HA-DW converge to a higher performance plateau in accuracy and acquired higher reward compared to the original RL algorithms, suggesting that the application of HA-DW boosts the exploration of hard prompts and weakens the exploitation of easy ones by mitigating the biased advantage estimation. All methods were trained until reaching stable performance plateaus; Figure~\ref{fig:training-dynamics} shows only a truncated portion of the full trajectories, while the displayed range already contains the best checkpoint of each method. The reward curves have also clearly plateaued, indicating that HA-DW's gains persist near convergence rather than merely reflecting faster early-stage optimization. In addition, our method encourages longer reasoning, greatly improving its reasoning abilities \citep{DBLP:conf/acl/JinYSZHMZD24,DBLP:journals/corr/abs-2501-12948}. HA-DW is capable of incentivizing the model to produce more sophisticated reasoning chain of thoughts to tackle more challenging tasks.

\paragraph{Ablation Study on $C_t$.}
We evaluate the effectiveness of the dynamic threshold $C_t$ by comparing it
with a fixed thresholds across five mathematical benchmarks, as shown in
Table~\ref{ablation study}. Experiments on Qwen3-4B-Base with GRPO-based training
show that dynamic adjustment achieves the best performance. Removing $C_t$
degrades performance, while a fixed threshold still improves over the baseline
by partially mitigating biased estimation. By incorporating cross-batch
information, $C_t$ captures long-term reward signals and further enhances RL
performance.

\paragraph{Supplementary Experiments (Appendix~\ref{appendix:supplementary})}
Due to space limitations, we include the following additional experiments in
Appendix~\ref{appendix:supplementary}:  
(1) empirical verification of advantage estimation bias,  
(2) an ablation study on the group size $G$ (Appendix \ref{appendix:different_group_size}),   
(3) an ablation study on the scaling parameter $\lambda_{\mathrm{scale}}$, 
(4) experiments extended to instruction following and tool use.
(5) effect on easy prompts, 
(6) an ablation study on rollout token budget
(7) robustness to reward noise, and
(8) significance analysis.

\section{Related Work}

\noindent \textbf{GRPO and GRPO Variants.} Following the success of Deepseek-R1 \citep{DBLP:journals/corr/abs-2501-12948}, GRPO has attracted widespread attention. To achieve better performance, numerous GRPO-based variants have been proposed.  However, these variants adopt static prompt difficulty and suffer from insufficient exploration of model's capability. More related work are placed in Appendix \ref{appendix:relatedwork}.


\section{Conclusion}
Our work characterizes difficulty-dependent conditional estimation bias
in group-relative advantages within non-degenerate groups. To address this issue, we propose \emph{HA-DW}, which dynamically
adjusts advantage weights based on the model’s evolving state. Extensive
experiments demonstrate that \emph{HA-DW} effectively improves reasoning
performance by mitigating biased advantage estimation.

\newpage
\begin{ack}
This research was supported by China NSFC (No. 62506024). The authors declare no competing interests.
\end{ack}

\bibliographystyle{plainnat}
\bibliography{reference}


\clearpage
\appendix

\section*{Appendix}
\addcontentsline{toc}{section}{Appendix}

\section{Limitations}

This work reveals an intrinsic limitation of group-relative RL—namely, biased advantage estimation under non-degenerate sampling—and proposes HA-DW to effectively mitigate this issue.  Our study primarily focuses on the issue of group-wise estimation bias, restricting the application of HA-DW to group-relative methods. Nevertheless, estimation bias is pervasive, and future work will focus on extending this concept to a broader scope. HA-DW is theoretically applicable to tasks beyond standard RLVR. However, the analysis and experiments in this work are mainly conducted under the RLVR setting for reasoning tasks, such as mathematical reasoning and code generation. Its potential applications in broader domains remain to be further explored.

\section{Future Work}
Future work will explore extending HA-DW beyond mathematical reasoning to broader RLVR settings, including agentic tasks, long-horizon decision making, and environments with non-binary or noisy reward signals. As group-relative optimization is increasingly adopted in emerging RL paradigms, HA-DW may serve as a general difficulty-aware reweighting module that can be integrated with these methods to improve training efficiency and optimization stability. Another promising direction is to develop more robust difficulty estimation under imperfect verifiers and to investigate mechanisms that can additionally recover useful training signals from degenerate all-correct or all-incorrect groups.

\section{More Related Work} \label{appendix:relatedwork}
\noindent \textbf{Group-based RLVR.}
Artificial intelligence has achieved significant advances in recent years\citep{zou2025transformer,he2025llm,qu2025t,li2026counterfactual,qu2026spot,huang2026adaptive,liu2026cdrrm}.
Recent studies have proposed numerous improvements to group-based reinforcement learning algorithms\cite{huang2026does,lu2026contextual,li2026adaptive,wang2026policy}. DaGRPO \citep{xie2025dagrpo} tackles GRPO’s instability and poor sample efficiency (caused by low distinctiveness in on-policy rollouts) by introducing sequence-level gradient rectification to filter low-distinctiveness pairs and off-policy anchor augmentation to restore learning signals on hard prompts. To address the advantage reversion and advantage mirror issues of fixed advantage formulations in GRPO that fail to adapt to samples with varying trajectory certainty. LSPO \citep{DBLP:journals/corr/abs-2510-01459} adopts length-aware dynamic sampling to retain shortest/longest responses, addressing the ineffectiveness of RLVR training for LLM reasoning. GMPO \citep{DBLP:journals/corr/abs-2507-20673} uses the geometric mean of token-level rewards (replacing GRPO’s arithmetic mean) to resolve unstable policy updates from outlier importance sampling ratios. FAPO \citep{DBLP:journals/corr/abs-2510-22543} uses a generative reward model (GenRM) to detect flawed-positive rollouts and a parameter-free reward penalty, addressing unreliable reasoning patterns and performance limitations caused by such rollouts in RLVR. F-GRPO \citep{plyusov2026f} extends GRPO with focal-style reweighting to down-weight easy prompts, reduce overfitting to common solutions, and preserve rare correct ones. VESPO \citep{shen2026vespo} stabilizes off-policy LLM training through variational sequence-level soft importance reweighting, avoiding length-normalization bias while remaining robust to stale and asynchronous rollouts. KTAE \citep{DBLP:journals/corr/abs-2505-16826} uses statistical analysis to quantify tokens’ association with correct rollouts and combines it with rollout-level advantages, solving the coarse granularity issue of GRPO that ignores token-specific contributions. Dr.GRPO \citep{DBLP:journals/corr/abs-2503-20783} further revisits GRPO-style reinforcement learning by removing the response-length normalization term and adopting a simpler group-relative objective, showing that such normalization choices can substantially affect training stability and final reasoning performance. \citep{zhang2026heterogeneous} investigates advantage estimation in heterogeneous multi-agent collaborative reinforcement learning to facilitate effective knowledge transfer and collaborative policy optimization across agents.

\noindent \textbf{Leverage of Cross-batch Signals.} Cross-batch signals have found widespread application across numerous domains. 
SamS \cite{huang2026adaptive} leverages a cross-batch scheduler to identify high-quality samples for direct preference optimization.
XBM \citep{DBLP:conf/cvpr/WangZHS20} improves embedding learning by leveraging memory from previous batches to enhance the consistency and quality of embeddings. CBNS \citep{DBLP:conf/sigir/WangZH21} introduces a method to improve negative sampling in embedding learning by utilizing negative samples from different batches, enhancing the model's ability to learn more robust and generalized representations. CIBN \citep{DBLP:conf/cvpr/Yao0ZHL21} extends traditional batch normalization across iterations, rather than within a single batch, to improve model convergence and generalization. CBRL \citep{DBLP:conf/mm/YangLC16} utilizes reference samples from different batches during training to improve the learning of deep classification and retrieval models. CBHEM-PLB \citep{DBLP:journals/tip/TanLWLLGL22} combines cross-batch hard example mining with a pseudo large batch strategy to improve face recognition models.\\

\noindent \textbf{Biased Estimation.} Considerable research effort has been directed towards addressing the critical challenge of biased estimation. The Bias–Variance Tradeoff theory \citep{ DBLP:books/lib/HastieTF09, DBLP:books/lib/Murphy12} suggests that as a model's complexity increases, its bias decreases but its variance increases, and vice versa. It emphasizes that there is a balance between bias and variance that affects the overall error in model predictions, and finding the optimal model complexity is crucial to minimize both bias and variance. Retrace \citep{DBLP:conf/nips/MunosSHB16} addresses the challenge of bias estimation in off-policy reinforcement learning. It proposes a retracing technique to mitigate the bias caused by off-policy data, which can lead to inaccurate value estimates. V-trace \citep{DBLP:conf/icml/EspeholtSMSMWDF18, DBLP:books/daglib/0035704} introduces a method for improving off-policy reinforcement learning by applying importance-weighted corrections to the value function updates in actor-critic algorithms, mitigating bias in off-policy data. DR-OVR \citep{DBLP:conf/icml/JiangL16, DBLP:conf/iclr/TschannenDRGL20} combines importance sampling and regression to correct for bias in off-policy value estimation, making it more stable and accurate. 

\section{Detailed Instantiations for GRPO and Related Algorithms}\label{appendix:loss func}

In this section, we present detailed instantiations of three group-relative reinforcement learning algorithms: GRPO, GSPO, and DAPO. And $t$ and $\tau$ denote training step and token index in this part.

\par GRPO streamlines PPO by discarding the value network without compromising stability. Instead of fitting a baseline, it derives the advantage using group-relative normalization. This group-normalized advantage is then assigned uniformly to all tokens in the response, formulating the clipped surrogate loss:
\begin{equation}
    \begin{aligned}
        J_{\text{GRPO}}(\theta) = &\frac{1}{G} \sum_{i=1}^{G} \frac{1}{|o_{t,i}|} \sum_{\tau=1}^{|o_{t,i}|} \min \left( r_{t,i,\tau}(\theta)\hat{A}_{t,i,\tau}, \right. \left. \text{clip}\left(r_{t,i,\tau}(\theta), 1-\epsilon, 1+\epsilon\right) \hat{A}_{t,i,\tau} \right),
    \end{aligned}
\end{equation}
where $\epsilon$ is the clipping hyperparameter and $r_{t,i,\tau}$ is is the importance sampling ratio comparing the new and old policy:
\begin{equation}
    r_{t,i,\tau}(\theta)=\frac{\pi_{\theta}(y_{t,i,\tau} \mid x_t, y_{t,i,<\tau})}{\pi_{\theta_{\text{old}}}(y_{t,i,\tau} \mid x_t, y_{t,i,<\tau})}.
\end{equation}
And GRPO defines the group advantage by subtracting the average reward of the group and normalizing by its standard deviation:
\begin{equation}
    \hat{A}_{t,i,\tau} = \frac{R(x_t, o_{t,\tau}) - \mathrm{mean}\left(\{R(x_t, o_{t,j})\}_{j=1}^G\right)}{\mathrm{std}\left(\{R(x_t, o_{t,j})\}_{j=1}^G\right)},
\end{equation}
where $R(x,o)$ denotes the reward function.

The objective function of GRPO applied with HA-DW can be denoted as:
\begin{equation}
    \begin{aligned}
        &J_{\text{GRPO+HA-DW}}(\theta) \\
        &= \frac{1}{G} \sum_{i=1}^{G} \frac{1}{|o_{t,i}|} \sum_{\tau=1}^{|o_{t,i}|} \min \left( r_{t,i,\tau}(\theta)\hat{A}_{t,i,\tau}\cdot \Phi_{t,i}, \right. 
        \left. \text{clip}\left(r_{t,i,\tau}(\theta), 1-\epsilon, 1+\epsilon\right) \hat{A}_{t,i,\tau}\cdot \Phi_{t,i} \right),
    \end{aligned}
\end{equation}
where $\Phi_{t,i}$ is the history-aware reweighting factor defined earlier.
\par GSPO optimizes policy learning by defining importance ratios at the sequence level, eliminating the need for a critic model. Rather than relying on a separate value network, it computes advantages through normalized relative rewards of group responses. This sequence-level advantage is directly used for policy updates without token-level processing, yielding the following objective function:
\begin{equation}
    \begin{aligned}
        J_{\text{GSPO}}(\theta) = &\frac{1}{G} \sum_{i=1}^{G} \min\bigg( 
    r_{t,i}(\theta) \hat{A}_{t,i},\  
    \text{clip}\bigg(r_{t,i}(\theta),\ 1-\epsilon,\ 1+\epsilon\bigg) \hat{A}_{t,i} 
\bigg),
    \end{aligned}
\end{equation}
where the sequence-level importance sampling ratio $r_i(\theta)$ can be denoted as:
\begin{equation}
\begin{aligned}
    r_{t,i}(\theta) &= \frac{\pi_\theta(y_{t,i} | x_t)} {\pi_{\theta_{\text{old}}}(y_{t,i} | x_t)} 
    = \frac{\prod_{\tau=1}^{|y_{t,i}|} \pi_\theta(y_{t,i,\tau} | x_t, y_{t,i,<\tau})}{\prod_{t=1}^{|y_{t,i}|} \pi_{\theta_{\text{old}}}(y_{t,i,\tau} | x_t, y_{t,i,<\tau})},
\end{aligned}
\end{equation}
where the advantage for GSPO can be denoted as:
\begin{equation}
    \hat{A}_{t,i} = \frac{R(x_t, o_{t,i}) - \mathrm{mean}\left(\{R(x_t, o_{t,j})\}_{j=1}^G\right)}{\mathrm{std}\left(\{R(x_t, o_{t,j})\}_{j=1}^G\right)}
\end{equation}
And the objective function of GSPO+HA-DW is:
\begin{equation}
    \begin{aligned}
        &J_{\text{GSPO+HA-DW}}(\theta)\\
        &= \frac{1}{G} \sum_{i=1}^{G} \min\bigg( 
    r_{t,i}(\theta) \hat{A}_{t,i}\cdot \Phi_{t,i} ,\  
    \text{clip}\bigg(r_{t,i}(\theta),\ 1-\epsilon,\ 1+\epsilon\bigg) \hat{A}_{t,i}\cdot \Phi_{t,i} 
\bigg).
    \end{aligned}
\end{equation}

DAPO’s key feature is operating at the token level instead of treating full responses as single units, ensuring each token in sampled output $o_i$ contributes proportionally to gradient updates. This fine-grained optimization boosts training stability and delivers more informative feedback for LLMs. The objective function is defined as:
\begin{equation}\label{eq:dapo-loss}
    \begin{aligned}
        &J_{\text{DAPO}}(\theta)  
        =\frac{1}{\sum_{i=1}^{G} |o_{t,i}|} \sum_{i=1}^{G} \sum_{\tau=1}^{|o_{t,i}|} \min \bigg( r_{t,i,\tau}(\theta) \hat{A}_{t,i,\tau}, 
        \text{clip}\big(r_{t,i,\tau}(\theta), 1 - \epsilon, 1 + \epsilon'\big) \hat{A}_{t,i,\tau} \bigg) .
    \end{aligned}
\end{equation}

DAPO introduces two key mechanisms: decoupled clipping and dynamic sampling, to address the limitations of traditional group-based methods. Decoupled clipping refines the trust region for more stable updates, while dynamic sampling mitigates estimation bias by adaptively reweighting samples based on their distribution.

Applying HA-DW on Equation~\eqref{eq:dapo-loss}, and we have:
\begin{equation}
    \begin{aligned}
        &J_{\text{DAPO+HA-DW}}(\theta)  \\
        &=\frac{1}{\sum_{i=1}^{G} |o_{t,i}|} \sum_{i=1}^{G} \sum_{\tau=1}^{|o_{t,i}|} \min \bigg( r_{t,i,\tau}(\theta) \hat{A}_{t,i,\tau}\cdot \Phi_{t,i}, 
        \text{clip}\big(r_{t,i,\tau}(\theta), 1 - \epsilon, 1 + \epsilon'\big) \hat{A}_{t,i,\tau}\cdot \Phi_{t,i} \bigg) .
    \end{aligned}
\end{equation}

\section{Setup Details}
\label{appendix:training-hyperparameters}

\textbf{Models \& Datasets.} We conduct our experiments on Qwen3-4B-Base, Qwen3-8B-Base \cite{qwen3technicalreport} and LLaMA-3.2-3B-Instruct to assess the mathematical reasoning performance of different algorithms across models of varying scales and family. Our training dataset is sourced from MATH dataset \citep{DBLP:conf/nips/HendrycksBKABTS21,DBLP:conf/iclr/LightmanKBEBLLS24} which contains 7.5$k$ questions for training. Our evaluation suite includes: MATH500 \citep{DBLP:conf/nips/HendrycksBKABTS21}, AMC23, AIME25, Minerva, and OlympiadBench \citep{DBLP:conf/acl/HeLBHTSHHHZLQL024}. To mitigate high variance on small benchmark sets and obtain reliable results, we report avg@16 on AIME25 and AMC23. \\

\noindent \textbf{Baseline.} We apply our proposed method on top of several representative group-relative reinforcement learning algorithms: GRPO, GSPO, and DAPO. We compare the performance of group-relative algorithms applying HA-DW to original ones, verifying the effectiveness and scalability of our method. \\

\noindent \textbf{Implementation Details.} We conduct RL training within the VeRL framework \citep{sheng2024hybridflow} on a single node with 8 $\times$ NVIDIA A100 GPUs. All experiments use a maximum prompt batch size of 1,024 and a maximum response length of 4,096. More hyperparameter details are provides in appendix \ref{appendix:training-hyperparameters}.

\paragraph{Training Hyperparameters.} The detailed hyperparameters used during our training process on 6 different methods of 3 models (Qwen3-4B-Base, Qwen3-8B-Base and LLaMA-3.2-3B-Instruct) used in our experiments are demonstrated in Table~\ref{table:hyperparameter}.

\section{Theoretical Proof}\label{appendix:theoretical proof}
\subsection{Proof of Theorem~\ref{theo:main1}} \label{proof of theorem1}
Fix a prompt and policy under the assumptions of
Theorem~\ref{theo:main1}. We analyze the non-degenerate event
\begin{equation*}
    \mathcal{S}\coloneqq \{1\le R\le G-1\}.
\end{equation*}

Under the retention condition $\mathcal{S}$, $\mathbb{E}\left[\hat{p}_t\mid S\right]$ denotes the conditional expectation of the empirical estimation $\hat{p}_t=R / G$. And it can be derived what the relationship is between it and the expected reward $p_t$:
\begin{equation}
\begin{aligned}
    \mathbb{E}\left[\hat{p}_t \mid \mathcal{S}\right] 
    &= \mathbb{E}\left[\frac{R}{G} \mid \mathcal{S}\right] \\
    &= \frac{1}{G} \cdot \frac{\mathbb{E}\left[R \cdot \mathbf{1}_{\{\mathcal{S}\}}\right]}{\mathbb{P}(\mathcal{S})} \\
    &= \frac{1}{G} \cdot \frac{\mathbb{E}[R] - \mathbb{E}\left[R \cdot \mathbf{1}_{\{R=G\}}\right]}{\mathbb{P}(\mathcal{S})} \\
    &= \frac{1}{G} \cdot \frac{Gp_t - G\mathbb{P}(R=G)}{\mathbb{P}(\mathcal{S})} \\
    &= \frac{p_t - p_t^G}{1 - (1-p_t)^G - p_t^G},
\end{aligned}
\end{equation}
where the indicator function $\mathbf{1}_{\{\mathcal{S}\}}$ takes the value $1$ if the event $\mathcal{S}$ occurs and $0$ otherwise. Through the conditional expectation of $\hat{p}_t$, we can obtain that its expected value is more than $p_t$ when $p_t<\frac{1}{2}$ and the baseline tends to be  overestimated. Conversely, when $p_t>\frac{1}{2}$, the expected value is smaller than $p_t$, leading to an underestimation.

For every sampled group, $\sum_{i=1}^G\hat{A}_{t,i}=0$.
The rewards are exchangeable and $\mathcal{S}$ is permutation-invariant,
so $\mathbb{E}[\hat{A}_{t,i}\mid\mathcal{S}]=0$.
The samplewise identity
\[
\hat{A}_{t,i}-A_{t,i}
=(r_{t,i}-\hat{p}_t)-(r_{t,i}-p_t)=p_t-\hat{p}_t
\]
then gives
\[
\begin{aligned}
&\mathbb{E}[\hat{A}_{t,i}-A_{t,i}\mid\mathcal{S}]\\
&\quad=p_t-\mathbb{E}[\hat{p}_t\mid\mathcal{S}]\\
&\quad=p_t-\frac{p_t-p_t^G}{1-(1-p_t)^G-p_t^G}\\
&\quad=\frac{p_t(1-p_t)
\bigl[p_t^{G-1}-(1-p_t)^{G-1}\bigr]}
{1-(1-p_t)^G-p_t^G}.
\end{aligned}
\]
The denominator and $p_t(1-p_t)$ are positive because
$G\ge2$ and $p_t\in(0,1)$.
The remaining factor is negative, zero, or positive according as
$p_t<1/2$, $p_t=1/2$, or $p_t>1/2$.
This proves Theorem~\ref{theo:main1}.
Without conditioning on $\mathcal{S}$,
$\mathbb{E}[\hat{p}_t]=p_t$, and therefore
$\mathbb{E}[\hat{A}_{t,i}-A_{t,i}]=0$.

\begin{lemma} \label{lemma:biasdetails}
Under the condition of Theorem \ref{theo:main1}, the bias induced by the group-relative advantage is formulated as:
\begin{equation}
\begin{aligned}
     & \mathbb{E}\left[A_{t,i}-\hat{A}_{t,i}\mid\mathcal{S}\right]    
     = \frac{p_t(1-p_t)^G + p_t^{G+1} - p_t^G}{1 - (1-p_t)^G - p_t^G}.
    \end{aligned}
\end{equation}
\end{lemma}
\begin{proof}
\begin{equation}
\begin{aligned}
     &\mathbb{E}\left[\hat{p}_t \mid \mathcal{S}\right]  - p_t  
     = \frac{p_t(1-p_t)^G + p_t^{G+1} - p_t^G}{1 - (1-p_t)^G - p_t^G}.
    \end{aligned}
\end{equation}
The identity $A_{t,i}-\hat{A}_{t,i}=\hat{p}_t-p_t$ completes the proof.
\end{proof}

\subsection{Proof of Corollary~\ref{cor:reward_conditional}}
\label{proof:reward_conditional}

Under the assumptions of Theorem~\ref{theo:main1}, the reward-conditioned
advantages satisfy
\begin{equation}
\label{eq:reward_conditional_advantages}
\begin{aligned}
\mathbb{E}[\hat{A}_{t,i}\mid r_{t,i}=1,\mathcal{S}]
&=\frac{(G-1)(1-p_t)}{G(1-p_t^{G-1})},\\
\mathbb{E}[\hat{A}_{t,i}\mid r_{t,i}=0,\mathcal{S}]
&=-\frac{(G-1)p_t}{G\bigl(1-(1-p_t)^{G-1}\bigr)}.
\end{aligned}
\end{equation}

\begin{proof}
Let $n=G-1$, and let $K$ be the number of correct responses among
the other $n$ samples. Conditional on either value of $r_{t,i}$,
independence gives $K\sim\mathrm{Binomial}(n,p_t)$.
If $r_{t,i}=1$, then $\mathcal{S}$ is equivalent to $K\le n-1$,
and $\hat{A}_{t,i}=(n-K)/G$.
Since $n-K=0$ when $K=n$,
\[
\mathbb{E}[\hat{A}_{t,i}\mid r_{t,i}=1,\mathcal{S}]
=\frac{\mathbb{E}[n-K]}{G(1-p_t^n)}
=\frac{n(1-p_t)}{G(1-p_t^n)}.
\]
If $r_{t,i}=0$, then $\mathcal{S}$ is equivalent to $K\ge1$,
and $\hat{A}_{t,i}=-K/G$.
The term $K=0$ contributes nothing to $\mathbb{E}[K]$, so
\[
\mathbb{E}[\hat{A}_{t,i}\mid r_{t,i}=0,\mathcal{S}]
=-\frac{\mathbb{E}[K]}{G[1-(1-p_t)^n]}
=-\frac{np_t}{G[1-(1-p_t)^n]}.
\]
Subtracting the corresponding targets yields
\[
\begin{aligned}
&\mathbb{E}[\hat{A}_{t,i}\mid r_{t,i}=1,\mathcal{S}]-(1-p_t)\\
&\qquad=\frac{(1-p_t)(Gp_t^{G-1}-1)}{G(1-p_t^{G-1})},\\
&\mathbb{E}[\hat{A}_{t,i}\mid r_{t,i}=0,\mathcal{S}]+p_t\\
&\qquad=\frac{p_t[1-G(1-p_t)^{G-1}]}
{G[1-(1-p_t)^{G-1}]}.
\end{aligned}
\]
Since $G\le2^{G-1}$ for $G\ge2$, the first expression is negative
when $p_t<1/2$, and the second is positive when $p_t>1/2$.
\end{proof}

\subsection{Derivation of the Difficulty-Gap Formulation}
\label{proof:difficulty_gap}

\begin{proof}
Substituting $p_t=1/2-d$ into the conditional mean error in
Theorem~\ref{theo:main1} gives $-B_G(d)$.
Exchanging $p_t$ and $1-p_t$ reverses its sign, so the value at
$p_t=1/2+d$ is $B_G(d)$.
For $d\in[0,1/2)$, the numerator of $B_G(d)$ factors as
\[
(\tfrac12-d)(\tfrac12+d)
\left[(\tfrac12+d)^{G-1}-(\tfrac12-d)^{G-1}\right]\ge0.
\]
Its denominator is positive. Thus the absolute conditional mean
error equals $B_G(d_t)$ and $B_G(0)=0$.

For fixed $G$, let $N_G(d)$ and $D_G(d)$ denote the numerator
and denominator of $B_G(d)$.
The numerator is odd, the denominator is even, and
\[
N_G'(0)=(G-1)2^{1-G},
\qquad D_G(0)=1-2^{1-G}>0.
\]
Hence
\[
\begin{aligned}
N_G(d)&=(G-1)2^{1-G}d+O_G(d^3),\\
D_G(d)&=1-2^{1-G}+O_G(d^2).
\end{aligned}
\]
Dividing gives
\[
B_G(d)=\frac{G-1}{2^{G-1}-1}d+O_G(d^3).
\]
Finally, for $a_G=(G-1)/(2^{G-1}-1)$,
\[
a_G-a_{G+1}
=\frac{(G-2)2^{G-1}+1}{(2^{G-1}-1)(2^G-1)}>0.
\]
Therefore, the leading coefficient decreases with $G$.
\end{proof}

\subsection{Proof of Theorem~\ref{theo:hard regime lower bound} and Corollary~\ref{cor:prob_bias}}
For hard prompts, in Theorem~\ref{theo:hard regime lower bound}, we have:
\begin{equation}
    \begin{aligned}
    &\mathbb{P}\left(\hat{p}_t - p_t > \epsilon \mid \mathcal{S}\right) 
    = \frac{\sum_{k=\lfloor G\left(p_t+\epsilon\right) \rfloor + 1}^{G-1} \binom{G}{k} \, p_t^{k} (1-p_t)^{G-k}}{1 - (1-p_t)^G - p_t^G} .
\end{aligned}
\end{equation}
The above equation is given by the following argument:
The conditioning event $\mathcal{S}$ restricts the sample space by excluding the outcome $R\in \{0, G\}$
(hence under $\mathcal{S}$ we only keep $R\in\{1,\dots,G-1\}$). Let:
\begin{equation}
    m(p_t)\coloneqq \left\lfloor G(p_t+\epsilon)\right\rfloor+1.
\end{equation}
Therefore, within the event $\mathcal{S}$, the deviation event $A$ becomes:
\begin{equation}
    \begin{aligned}
        A\cap \mathcal{S}
        &=\{R \ge m(p_t)\}\cap\{1\leq R\le G-1\} 
        =\{m(p_t)\le R\le G-1\}.
    \end{aligned}
\end{equation}

By definition of conditional probability, the {numerator} is the (unconditional) probability mass of all outcomes that
satisfy the deviation requirement $\hat p-p>\epsilon$ and simultaneously, and satisfy the restriction imposed by $\mathcal{S}$. Because $R$ is binomial, for any integer $k$ we have:
\begin{equation}
    \mathbb{P}(R=k)=\binom{G}{k}p_t^k(1-p_t)^{G-k}.
\end{equation}
Summing over all admissible counts $k\in\{m(p_t),m(p_t)+1,\dots,G-1\}$ yields:
\begin{align}
    \mathbb{P}(A\cap \mathcal{S})
=
\sum_{k=m(p_t)}^{G-1} \mathbb{P}(R=k)
=
\sum_{k=m(p_t)}^{G-1}\binom{G}{k}p_t^k(1-p_t)^{G-k}.
\end{align}

Thus, based on the formula of conditional probability \citep{DBLP:conf/cikm/WuYLD25, DBLP:journals/corr/abs-2502-00949}, we can derive the conclusion of Theorem~\ref{theo:hard regime lower bound}.

According to Theorem~\ref{theo:hard regime lower bound}, we can formulate:
\begin{equation}
    f(G, p_t)\coloneqq \mathbb{P}(\hat{p}_t - p_t > \epsilon \mid \mathcal{S}).
\end{equation}

Assume that $p_t$ follows a uniform distribution. And we define:
\begin{equation}\label{eq:jifen}
    \begin{aligned}
        \mathbb{P}(G,p_{t_1},p_{t_2}) 
    &\coloneqq \frac{1}{p_{t_2}-p_{t_1}} \int_{p_{t_1}}^{p_{t_2}} f(G, p_t) dp_t
    \end{aligned}
\end{equation}
where $p_{t_1}$ and $p_{t_2}$ is the expected reward. And $\mathbb{P}(G)$ reflects the probability that, when $G$ is fixed, the baseline $\hat{p}_t$ is overestimated of group-relative RL algorithms over a certain expected-reward interval. For hard prompts with $p_t\in \left(0,0.25\right)$ under different group size $G$, when $G\in[2,8]$, we have:

\begin{table}[ht]
\centering
\renewcommand{\arraystretch}{1.25}
\setlength{\tabcolsep}{32pt} 

\begin{tabular}{c c}
\hline
$G$ & $\mathbb{P}(G,0,0.25)$ \\
\hline
2  & 0.999997499987 \\
4  &  0.999995948256 \\
6  & 0.827761785622 \\
\textbf{8}  & \textbf{0.781129955681} \\
\hline
\end{tabular}

\caption{$\mathbb{P}(G,0,0.25)$ as a function of $G\in [2,8]$.}
\label{tab:P_vs_G_1}
\end{table}

Similarly, we can calculate hard prompts with $p_t \in (0,0.5)$ under different group size $G$.
\begin{table}[ht]
\centering
\renewcommand{\arraystretch}{1.25}
\setlength{\tabcolsep}{32pt} 
\begin{tabular}{c c}
\hline
$G$ & $\mathbb{P}(G,0,0.5)$ \\
\hline
2  & 0.999994999975 \\
4  &  0.776965795853 \\
6  & 0.689721502158 \\
\textbf{8}  & \textbf{0.640944744224} \\
\hline
\end{tabular}

\caption{$\mathbb{P}(G,0,0.5)$ as a function of $G\in [2,8]$.}
\label{tab:P_vs_G_2}
\end{table}

We can conclude from Table~\ref{tab:P_vs_G_1}, when $2\le G\le 8$, $\mathbb{P}(G,0,0.25)>0.78$. This reveals that for hard prompts whose $p_t\in (0,0.25)$ when $G$ is limited, its baseline $\hat{p}_t$ of group-relative RL algorithms is substantially likely to be overestimated. Similarly, due to the evident symmetry of the group-relative methods, for easy prompt with $p_t\in (0.75,1)$, the baseline $\hat{p}_t$ is underestimated with the same probability distribution. 

Based on the aforementioned conclusions, for group-based algorithms, when $G\in [2,8]$, the probability of biased advantage estimation can be denoted as:
\begin{equation}
\begin{aligned}
\mathbb{P}\!\left(\hat{A}_{t,i} < A_{t,i} \mid \mathcal{S},\; p_t < 0.25 \right)
&> 0.78, \\
\mathbb{P}\!\left(\hat{A}_{t,i} > A_{t,i} \mid \mathcal{S},\; p_t > 0.75 \right)
&> 0.78.
\end{aligned}
\end{equation}
Similarlly, Table \ref{tab:P_vs_G_2} can give: 
\begin{equation}
\begin{aligned}
\mathbb{P}\!\left(\hat{A}_{t,i} < A_{t,i} \mid \mathcal{S},\; p_t < 0.5 \right)
&> 0.63, \\
\mathbb{P}\!\left(\hat{A}_{t,i} > A_{t,i} \mid \mathcal{S},\; p_t > 0.5 \right)
&> 0.63.\\
\end{aligned}
\end{equation}

For $ p_t > 0.875$ and $ p_t < 0.125$, the results are the adaptation of Corollary \ref{cor:extreme_sure}.

\subsection{Conditional Concentration and Bias Rates}
\label{proof:conditional_error_rates}

\begin{corollary}[Conditional Concentration and Bias Rates]
\label{cor:conditional_error_rates}
Under Theorem~\ref{theo:main1}'s assumptions, for $\epsilon>0$,
\[
\begin{aligned}
&\mathbb{P}\!\left(
|\hat{A}_{t,i}-A_{t,i}|\ge\epsilon\,\middle|\,\mathcal{S}
\right)\\
&\qquad\le\min\!\left\{1,\frac{2e^{-2G\epsilon^2}}
{1-(1-p_t)^G-p_t^G}\right\}.
\end{aligned}
\]
For $\delta\in(0,1)$, with conditional probability at least $1-\delta$,
\begin{equation}
\label{eq:conditional_error_radius}
|\hat{A}_{t,i}-A_{t,i}|
\le\sqrt{\frac{1}{2G}
\log\!\left(\frac{2}
{\delta\bigl[1-(1-p_t)^G-p_t^G\bigr]}\right)}.
\end{equation}
The conditional mean bias is derived below.
For fixed $\eta_{\mathrm{margin}}\in(0,1/2]$, uniformly over
$p_t\in[\eta_{\mathrm{margin}},1-\eta_{\mathrm{margin}}]$,
\[
\left|\mathbb{E}[\hat{A}_{t,i}-A_{t,i}\mid\mathcal{S}]\right|
=O_{\eta_{\mathrm{margin}}}\!
\left((1-\eta_{\mathrm{margin}})^G\right)
\quad\text{as }G\to\infty.
\]
If $p_t=O(G^{-1})$ or $1-p_t=O(G^{-1})$,
with $p_t\in(0,1)$ varying with $G$,
\[
\left|\mathbb{E}[\hat{A}_{t,i}-A_{t,i}\mid\mathcal{S}]\right|
=O(G^{-1}).
\]
\end{corollary}

For fixed $\delta$, the deviation bound in
Equation~\eqref{eq:conditional_error_radius} is $O(G^{-1/2})$
whenever the retention probability $1-(1-p_t)^G-p_t^G$ is bounded
away from zero.
A smaller retention probability weakens this concentration bound.
The conditional mean bias has the separate rates above and decays
exponentially in $G$ when $p_t$ stays a fixed distance from both
endpoints.
The proof follows.

\begin{proof}
Using $\hat{A}_{t,i}-A_{t,i}=p_t-\hat{p}_t$,
conditional probability and Hoeffding's inequality give
\[
\begin{aligned}
&\mathbb{P}\!\left(
|\hat{A}_{t,i}-A_{t,i}|\ge\epsilon\,\middle|\,\mathcal{S}
\right)\\
&\quad\le
\frac{\mathbb{P}(|\hat{p}_t-p_t|\ge\epsilon)}
{1-(1-p_t)^G-p_t^G}\\
&\quad\le\frac{2e^{-2G\epsilon^2}}{1-(1-p_t)^G-p_t^G}.
\end{aligned}
\]
Combining this with the upper bound $1$ proves the concentration
inequality. Setting
\[
\epsilon=\sqrt{\frac{1}{2G}
\log\!\left(\frac{2}{\delta[1-(1-p_t)^G-p_t^G]}\right)}
\]
gives Equation~\eqref{eq:conditional_error_radius}.

The conditional mean bias is
\[
\left|\mathbb{E}[\hat{A}_{t,i}-A_{t,i}\mid\mathcal{S}]\right|
=\frac{\left|p_t(1-p_t)^G-(1-p_t)p_t^G\right|}
{1-(1-p_t)^G-p_t^G}.
\]

For $p_t\le1/2$, the numerator of the absolute conditional mean
error is $p_t(1-p_t)^G-(1-p_t)p_t^G\ge0$, and
\[
1-(1-p_t)^G-p_t^G
=\mathbb{P}(\mathcal{S})
\ge\mathbb{P}(R=1)=Gp_t(1-p_t)^{G-1}.
\]
Therefore,
\[
\begin{aligned}
\left|\mathbb{E}[\hat{A}_{t,i}-A_{t,i}\mid\mathcal{S}]\right|
&\le\frac{p_t(1-p_t)^G}{Gp_t(1-p_t)^{G-1}}\\
&=\frac{1-p_t}{G}\le\frac1G.
\end{aligned}
\]
The case $p_t\ge1/2$ follows by symmetry.
This bound holds throughout $(0,1)$ and in particular establishes
the stated extreme-regime rate.

If $p_t\in[\eta_{\mathrm{margin}},1-\eta_{\mathrm{margin}}]$,
then, since $G\ge2$,
\[
\begin{aligned}
1-(1-p_t)^G-p_t^G
&\ge1-(1-p_t)^2-p_t^2\\
&=2p_t(1-p_t)\\
&\ge2\eta_{\mathrm{margin}}(1-\eta_{\mathrm{margin}}),
\end{aligned}
\]
while
\[
\begin{aligned}
&\left|p_t(1-p_t)^G-(1-p_t)p_t^G\right|\\
&\quad\le p_t(1-p_t)^G+(1-p_t)p_t^G\\
&\quad\le(1-\eta_{\mathrm{margin}})^G.
\end{aligned}
\]
Consequently,
\[
\begin{aligned}
\left|\mathbb{E}[\hat{A}_{t,i}-A_{t,i}\mid\mathcal{S}]\right|
&\le\frac{(1-\eta_{\mathrm{margin}})^G}
{2\eta_{\mathrm{margin}}(1-\eta_{\mathrm{margin}})}\\
&=O_{\eta_{\mathrm{margin}}}
\!\left((1-\eta_{\mathrm{margin}})^G\right),
\end{aligned}
\]
uniformly over the stated interval.
\end{proof}

\subsection{Proof of Corollary~\ref{corollary:sufficient G} and Corollary \ref{cor:extreme_sure}}\label{appendix:additional theoretical analysis}

Let $G$ be a large integer, for hard prompts, according to Theorem~\ref{theo:hard regime lower bound}, we have:
\begin{equation}
    \begin{aligned}
    &\mathbb{P}\left(\hat{p}_t - p_t > \epsilon \mid \mathcal{S} \right) 
    = \frac{\sum_{k=\lfloor G\left(p_t+\epsilon\right) \rfloor + 1}^{G-1} \binom{G}{k} \, p_t^{k} (1-p_t)^{G-k}}{1 - (1-p_t)^G - p_t^G} .
\end{aligned}
\end{equation}
And we define:
\begin{equation}
    f(p_t)\coloneqq \mathbb{P}(\hat{p}_t - p_t > \epsilon \mid \mathcal{S}).
\end{equation}

We analyze the integral in the limit of large $G$ using the Poisson approximation \citep{serfling1978some}. Let us perform the change of variable $x_t = Gp_t$. The limits of integration change from $[1/G, 2/G]$ to $[1, 2]$, and $dp_t = dx_t/G$. We define the integral of interest:
\begin{equation}
\begin{aligned}
    \mathbb{P}(G_1,G_2) &= \frac{G}{G_2-G_1} \int_{G_1}^{G_2} f(x_t/G) \frac{dx_t}{G} 
    = \int_{1}^{2} f(x_t/G) \, dx_t.
\end{aligned}
\end{equation}

First, we determine the summation lower bound $m(p_t)$. For $p_t \in [1/G, 2/G)$, we have $Gp_t \in [1, 2)$. Consequently, $\lfloor Gp_t \rfloor = 1$, which implies:
\begin{equation}
    m(p_t) = \lfloor Gp_t \rfloor + 1 = 2.
\end{equation}

Next, we approximate the binomial terms. In the limit $G \to \infty$ with $Gp_t = x$ fixed, the binomial distribution converges to the Poisson distribution with parameter $x_t$. The denominator $Z(p_t)$ approximates to:
\begin{equation}
    \begin{aligned}
        Z(p_t) &= 1 -(x_t/G)^G-  (1 - x_t/G)^G 
        \xrightarrow{G \to \infty} 1 - e^{-x_t}.
    \end{aligned}
\end{equation}

The numerator is the probability that a Poisson random variable $K \sim \text{Pois}(x_t)$ takes a value $k \ge 2$ (ignoring the upper limit $G-1$ as the Poisson tail vanishes exponentially):
\begin{equation}
    \begin{aligned}
        f(p_t) &=  \mathbb{P} (K\geq 2) = \sum_{k=2}^{\infty} \frac{x_t^k e^{-x_t}}{k!} \\
        &= 1 - \mathbb{P}(K=0) - \mathbb{P}(K=1) \\
        &= 1 - e^{-x_t} - x e^{-x_t} \\
        &= 1 - e^{-x_t}(1+x_t).
    \end{aligned}
\end{equation}

Substituting these approximations into $f(x_t/G)$, we obtain the limiting integrand $h(x_t)$:
\begin{equation}\label{29}
\begin{aligned}
    h(x_t) &= \frac{1 - e^{-x_t}(1+x_t)}{1 - e^{-x_t}} 
    = 1 - \frac{x_t e^{-x_t}}{1 - e^{-x_t}} = 1 - \frac{x_t}{e^{x_t} - 1}.
\end{aligned}
\end{equation}

Assume that $p_t$ follows a uniform distribution. Calculating Equation~\eqref{29} numerically, for sufficiently large $G$, we can show that $\mathbb{P}(0,2)=\frac{G}{2} \int_{0}^{2} f(x_t) \, dx_t =G\frac{\mathbb{P}(0,1)+\int_{1}^{2} h(x_t) \, dx_t}{2} \approx  0.7818$. 

Next, we use numerical computation to show how large $G$ should be, whose result can be found in the following table:
\begin{table}[ht]
\centering
\renewcommand{\arraystretch}{1.25}
\setlength{\tabcolsep}{32pt} 

\begin{tabular}{c c}
\hline
$G$ & $\frac{G}{2} \int_{0}^{2/G} f(p_t) \, dp_t$ \\
\hline
2  &  0.499997499987 \\
3  &   0.749995833315 \\
4  & 0.776965795853 \\
5  & 0.780787089465\\
\textbf{6}  & \textbf{0.781154327380} \\
\hline
\end{tabular}

\caption{$\frac{G}{2} \int_{0}^{2/G} f(p_t) \, dp_t$ as a function of $G\in [2,6]$.}
\label{tab:int_P_vs_G_3}
\end{table}

Thus, $G\geq 6$ is sufficiently large to have:
\begin{align}
    \mathbb{P}\!\left(\hat{A}_{t,i} < A_{t,i} \mid \mathcal{S},\; p_t < \frac{2}{G} \right)
> 0.78\,.
\end{align}

\textbf{Proof of Corollary \ref{cor:extreme_sure}}. On $S$, we have $R\ge 1$, hence $\hat p=R/G\ge 1/G$.
Since $p<1/G$, it follows that $\hat p\ge 1/G>p$. 
On $S$, we have $R\le G- 1$, hence $\hat p=R/G \leq (G-1)/G$.
Since $p>(G-1)/G$, it follows that $\hat p\le (G-1)/G<p$.
This leads to the Corollary.

\subsection{Proof of Lemma \ref{lemma:adjustment of baseline} and Theorem \ref{theo:algorithm}}\label{appendix:adjusted estimator}

\subsubsection{Proof of Lemma \ref{lemma:adjustment of baseline}}
Before proving Lemma~\ref{lemma:adjustment of baseline}, we begin with the following auxiliary result.                 
\begin{lemma} \label{lem:lem3}
Define the non-degenerate event $\mathcal S:=\{1\le S\le G-1\}$, and $\epsilon  \in (0, |p_t-\hat{p}_t|)$.
If 
\begin{equation}\label{eq:inequal}
    \begin{aligned}
        c \in & \left( \frac{\left(p_t - \epsilon\right) \cdot \left(1 - (1 - p_t)^G - p_t^G\right)}{p_t (1 - p_t^{G-1})}, \right. 
           \left. \frac{\left(p_t + \epsilon\right) \cdot \left(1 - (1 - p_t)^G - p_t^G\right)}{p_t (1 - p_t^{G-1})} \right),
    \end{aligned}
\end{equation}
we have:
\begin{equation}
   \mathbb{E}\!\left[\tilde{p}_t \mid \mathcal{S}\right]
    \in
    \left(p_t - \epsilon,\; p_t + \epsilon\right).
\end{equation}
\end{lemma}

\begin{proof}
We define the adjusted factor $c$ to compensate for the bias in the advantage estimation which applied on empirical group baseline $\hat{p}_t$. The globally scaled estimator can be approximated as:
\begin{equation}
\tilde{p}_t(R)\coloneqq c\, \hat{p}_t = c\,\frac{R}{G}.
\end{equation}

We can derive the conditional expectation of $\tilde{p_t}$ on non-degenerate event $\mathcal{S}= \{1\le R \le G-1\} $:
\begin{equation}
    \begin{aligned}
        \mathbb{E}[\tilde{p}_t \mid \mathcal{S}] &= \mathbb{E}\left[c \frac{R}{G} \,\middle|\,\mathcal{S}\right] 
        = \frac{c}{G} \mathbb{E}[R \mid \mathcal{S}] 
        = \frac{c}{G} \frac{\mathbb{E} [R \cdot \mathbf{1}_{\{\mathcal{S}\}}]}{\mathbb{P}(\mathcal{S})}.
    \end{aligned}
\end{equation}
And we have:
\begin{equation}
    \begin{aligned}
        \mathbb{E}[R \cdot \mathbf{1}_{\{S\}}] 
        &= \sum_{k=1}^{G-1} k \, \mathbb{P}(R = k) 
        = \mathbb{E}[R] - G \mathbb{P}(R = G).
    \end{aligned}
\end{equation}

Because the only term excluded from $\sum_{k=0}^{G} k\mathbb{P}(R=k)=\mathbb E[R]$ is the $k=G$ term (the $k=0$ term is zero anyway). Using $\mathbb E[R]=Gp_t$ and $\mathbb{P}(R=G)=p_t^G$, we can obtain:
\begin{equation}
    \begin{aligned}
        \mathbb{E}[R \cdot \mathbf{1}_{\{S\}}] &= Gp_t - Gp_t^G
        =Gp_t(1-p_t^{G-1}).
    \end{aligned}
\end{equation}
Therefore:
\begin{equation}
    \begin{aligned}
        \mathbb E[\tilde {p}_t \mid S]&=\frac{c}{G}\cdot \frac{Gp_t(1-p_t^{G-1})}{1-(1-p_t)^G-p_t^G}
        = c\,p_t\,\frac{1-p_t^{G-1}}{1-(1-p_t)^G-p_t^G},
    \end{aligned}
\end{equation}
which proves the stated conditional expectation formula.

\par To mitigate the biased estimation, let:
\begin{equation}
    \mathbb E[\tilde {p}_t \mid S]=p_t.
\end{equation}
And we can solve:
\begin{equation}
    c\,p_t\,\frac{1-p_t^{G-1}}{1-(1-p_t)^G-p_t^G}=p_t.
\end{equation}
The analytical solution for this equation is:
\begin{equation}
    c=\frac{1-(1-p_t)^G-p_t^G}{1-p_t^{G-1}}.
\end{equation}

When the adjustment coefficient $c$ falls within a specific range of values, we will have $| \tilde{p}_t-p_t|<| \hat{p}_t-p_t|$. We first let:
\begin{equation}
    \epsilon= | \hat{p}_t - p_t|.
\end{equation}

For $\mathbb{E}[\tilde{p_t} \mid S] = p_t + \epsilon$, solve:
\begin{equation}
    c_+\,p_t\,\frac{1-p_t^{G-1}}{1-(1-p_t)^G-p_t^G}=p_t+\epsilon.
\end{equation}
And we can derive:
\begin{equation}
    \begin{aligned}
        c_+ &= \frac{\left(p_t + \epsilon\right) \cdot \left(1 - (1 - p_t)^G - p_t^G\right)}{p_t (1 - p_t^{G-1})} 
        = \left(1 + \frac{\epsilon}{p_t}\right) c.
    \end{aligned}
\end{equation}

For $\mathbb{E}[\tilde{p_t} \mid S] = p_t - \epsilon$, solve:
\begin{equation}
    c_-\,p_t\,\frac{1-p_t^{G-1}}{1-(1-p_t)^G-p_t^G}=p_t-\epsilon.
\end{equation}
Thus, we have:
\begin{equation}
    \begin{aligned}
        c_- &= \frac{\left(p_t - \epsilon\right) \cdot \left(1 - (1 - p_t)^G - p_t^G\right)}{p_t (1 - p_t^{G-1})} 
        = \left(1 - \frac{\epsilon}{p_t}\right) c.
    \end{aligned}
\end{equation}

\begin{equation}\label{eq:inequal}
    \begin{aligned}
        c \in & \left( \frac{\left(p_t - \epsilon\right) \cdot \left(1 - (1 - p_t)^G - p_t^G\right)}{p_t (1 - p_t^{G-1})}, \right. 
           \left. \frac{\left(p_t + \epsilon\right) \cdot \left(1 - (1 - p_t)^G - p_t^G\right)}{p_t (1 - p_t^{G-1})} \right),
    \end{aligned}
\end{equation}
we have
\begin{equation}
   \mathbb{E}\!\left[\tilde{p}_t \mid \mathcal{S}\right]
    \in
    \left(p_t - \epsilon,\; p_t + \epsilon\right).
\end{equation}
\end{proof}

\begin{lemma}[$p_t$-free concentration under $\mathcal S$]
\label{lem:p_free_conditional_hoeffding}
Define the non-degenerate event $\mathcal S := \{1 \le S \le G-1\}$.
Assume $p_t \in [\eta_{\mathrm{margin}},\, 1-\eta_{\mathrm{margin}}]$ for some $\eta_{\mathrm{margin}} \in (0,1/2]$.
Then for any $\zeta>0$, we have:
\begin{equation}
\label{eq:p_free_conditional_bound}
\begin{aligned}
&\mathbb P\!\left(|\hat p_t - p_t| < \zeta \,\middle|\, \mathcal S\right)
\ge   
\frac{
1 - 2\exp(-2G \zeta^2) - (1-\eta_{\mathrm{margin}})^G - \eta_{\mathrm{margin}}^G
}{
1 - (1-\eta_{\mathrm{margin}})^G - \eta_{\mathrm{margin}}^G
}.
\end{aligned}
\end{equation}
\end{lemma}

\begin{proof}
    Let $A := \{|\hat p_t - p_t| < \zeta\}$.
By the definition of conditional probability:
\begin{equation}
\label{eq:cond_def}
\mathbb P(A \mid \mathcal S)
=
\frac{\mathbb P(A \cap \mathcal S)}{\mathbb P(\mathcal S)}.
\end{equation}
We lower bound the numerator. Since
\(
A \cap \mathcal S \supseteq A \setminus \mathcal S^{-},
\)
we have:
\begin{equation}
\label{eq:num_lb}
\mathbb P(A \cap \mathcal S)
\ge
\mathbb P(A) - \mathbb P(\mathcal S^{-}).
\end{equation}
Next, note that
\(
\mathcal S^{-} = \{S=0\} \cup \{S=G\}
\)
and these two events are disjoint. Therefore:
\begin{equation}\label{eq:Sc}
\begin{aligned}
    \mathbb P(\mathcal S^{-})
&=
\mathbb P(S=0) + \mathbb P(S=G)=
(1-p_t)^G + p_t^G.
\end{aligned}
\end{equation}
Moreover, we can derive that:
\begin{equation}
\label{eq:Sprob}
\mathbb P(\mathcal S)
=
1-\mathbb P(\mathcal S^{-})
=
1-(1-p_t)^G-p_t^G.
\end{equation}

We now lower bound $\mathbb P(A)$ using Hoeffding's inequality.
Since each $r_{t,i} \in [0,1]$ almost surely and $\{r_{t,i}\}_{i=1}^G$ are independent
with $\mathbb E[r_{t,i}]=p_t$, Hoeffding's inequality yields:
\begin{equation}
\label{eq:hoeffding_two_sided}
\mathbb P\!\left(|\hat p_t - p_t| \ge \zeta\right)
\le
2\exp(-2G\zeta^2),
\end{equation}
equivalently:
\begin{equation}
\label{eq:Aprob}
\mathbb P(A)
=
\mathbb P\!\left(|\hat p_t - p_t| < \zeta\right)
\ge
1-2\exp(-2G\zeta^2).
\end{equation}

It remains to remove the dependence on $p_t$ in $\mathbb P(\mathcal S)$.
Define $f(p):=p^G+(1-p)^G$. For $G\ge 1$, $f$ is symmetric around $1/2$ and
attains its maximum over $[\eta_{\mathrm{margin}},1-\eta_{\mathrm{margin}}]$ at the endpoints. Hence:
\begin{equation}
\label{eq:Sc_upper}
(1-p_t)^G+p_t^G = f(p_t)\le f(\eta_{\mathrm{margin}})=(1-\eta_{\mathrm{margin}})^G+\eta_{\mathrm{margin}}^G,
\end{equation}
which implies:
\begin{equation}
\label{eq:S_lower}
\mathbb P(\mathcal S)=1-f(p_t)\ge 1-(1-\eta_{\mathrm{margin}})^G-\eta_{\mathrm{margin}}^G.
\end{equation}

Combining Equation~\eqref{eq:cond_def} and \eqref{eq:num_lb} with Equation~\eqref{eq:Aprob}, \eqref{eq:Sc_upper}, and \eqref{eq:S_lower}, we can obtain that:
\begin{align}
&\mathbb P(A \mid \mathcal S) \\
&\ge 
\frac{\mathbb P(A)-\mathbb P(\mathcal S^{-})}{\mathbb P(\mathcal S)}
\notag\\
&\ge 
\frac{
\left[1-2\exp(-2G\zeta^2)\right]
-
\left[(1-\eta_{\mathrm{margin}})^G+\eta_{\mathrm{margin}}^G\right]
}{
1-(1-\eta_{\mathrm{margin}})^G-\eta_{\mathrm{margin}}^G
},
\end{align}
which completes the proof.
\end{proof}

\begin{lemma}[Conditional $p_t$-free concentration under $\mathcal S$]
\label{thm:cond_p_free_conc_fail}
Assume $p_t\in[\eta_{\mathrm{margin}},\,1-\eta_{\mathrm{margin}}]$ for some $\eta_{\mathrm{margin}}\in(0,1/2]$.
Then for any $\delta\in(0,1)$, with probability at least $1-\delta$
conditional on $\mathcal S$, we have:
\begin{equation}
\label{eq:c_from_delta_fail}
|\hat p_t-p_t|
<
\sqrt{\frac{1}{2G}\log\!\left(\frac{2}{\delta\big(1-(1-\eta_{\mathrm{margin}})^G-\eta_{\mathrm{margin}}^G\big)}\right)}.
\end{equation}
\end{lemma}
\begin{proof}
Now choose $\gamma$ such that the right-hand side of Equation~\eqref{eq:p_free_conditional_bound} is at
most $\delta$, i.e.:
\[
\frac{2\exp(-2G\gamma^2)}{1-(1-\eta_{\mathrm{margin}})^G-\eta_{\mathrm{margin}}^G}\le \delta.
\]
Solving for $\gamma$ gives:
\[
\gamma \ge
\sqrt{\frac{1}{2G}\log\!\left(
\frac{2}{\delta\big(1-(1-\eta_{\mathrm{margin}})^G-\eta_{\mathrm{margin}}^G\big)}
\right)}.
\]
Therefore, for
\[
\gamma^\star :=
\sqrt{\frac{1}{2G}\log\!\left(
\frac{2}{\delta\big(1-(1-\eta_{\mathrm{margin}})^G-\eta_{\mathrm{margin}}^G\big)}
\right)},
\]
we have
\(
\mathbb P(|\hat p_t-p_t|\ge \gamma^\star \mid \mathcal S)\le \delta
\),
equivalently,
\(
\mathbb P(|\hat p_t-p_t|< \gamma^\star \mid \mathcal S)\ge 1-\delta
\),
which proves Equation~\eqref{eq:c_from_delta_fail}.
\end{proof}

Finally, combining Lemma \ref{lem:lem3}, \ref{lem:p_free_conditional_hoeffding}, and \ref{thm:cond_p_free_conc_fail} gives Lemma \ref{lemma:adjustment of baseline}. We restated it here for completeness:
\begin{lemma}[A $p_t$-free feasible range of $c$ expressed via $\hat p_t$]
\label{lem:c_range_hatp}
Assume the conditions of Lemma~\ref{lem:p_free_conditional_hoeffding} and define:
\begin{equation}
\label{eq:eps_delta}
\epsilon_\delta
:=
\sqrt{\frac{1}{2G}\log\!\left(
\frac{2}{\delta\big(1-(1-\eta_{\mathrm{margin}})^G-\eta_{\mathrm{margin}}^G\big)}
\right)}.
\end{equation}
Let:
\begin{equation}
\begin{aligned}
&I_t := \bigl[\hat p_t-\epsilon_\delta,\ \hat p_t+\epsilon_\delta\bigr]\cap[\eta_{\mathrm{margin}},1-\eta_{\mathrm{margin}}], \quad
 A(p):=1-(1-p)^G-p^G.
\end{aligned}
\end{equation}
Fix any $\epsilon>0$,
we define:
\begin{equation}
\label{eq:cl_ch_def}
c_{\mathrm{low}}
:=
\sup_{p\in I_t}
\frac{(p-\epsilon)\,A(p)}{p(1-p^{G-1})},
\end{equation}
and:
\begin{equation}
c_{\mathrm{high}}
:=
\inf_{p\in I_t}
\frac{(p+\epsilon)\,A(p)}{p(1-p^{G-1})}.
\end{equation}
Then, on the event $\{|\hat p_t-p_t|<\epsilon_\delta\}$ (which holds with
probability at least $1-\delta$ conditional on $\mathcal S$), any choice
\begin{equation}
\label{eq:c_interval_hatp}
c \in (c_{\mathrm{low}},\ c_{\mathrm{high}})
\end{equation}
implies that the condition \eqref{eq:inequal} holds for the true $p_t$, and hence:
\[
\mathbb E[\tilde p_t\mid \mathcal S] \in (p_t-\epsilon,\ p_t+\epsilon).
\]
\end{lemma}

\subsubsection{Proof of Theorem \ref{theo:algorithm} }

\par When applying adjustment on the advantage $\hat{A}_{t,i}$, we do not consider the standard deviation here, and assume that:
\begin{equation}
    \begin{aligned}
        \Phi_{t,i}\hat{A}_{t,i}=r_{t,i}-\tilde{p}_t=r_{t,i}-c\hat{p}_t.
    \end{aligned}
\end{equation}
It is equivalent to:
\begin{equation}
    \Phi_{t,i}r_{t,i}-\Phi_{t,i}\hat{p}_t=r_{t,i}-c\hat{p}_t.
\end{equation}
And for correct responses with $r_{t,i}=1$:
\begin{equation}
    \Phi_{t,i}=\frac{1-c\hat{p}_t}{1-\hat{p}_t}.
\end{equation}
While for incorrect responses with $r_{t,i}=0$:
\begin{equation}
    \Phi_{t,i}=c.
\end{equation}

\par According to Equation~\eqref{eq:reweighting factor}:
\begin{equation}\label{eq:scale}
\begin{aligned}
    \Phi_{t,i}&=\lambda_\mathrm{scale}\cdot \exp\left(D_{t,i} \cdot M_{t}\right),
\end{aligned}
\end{equation}
the adjustment of $A_{t,i}$ can be categorized into four types. For responses in defined hard prompts with $r_{t,i}=1$, the adjusted advantage can be denoted as:
\begin{equation}
\hat{A}^{\mathrm{1}}_{t,i}=\lambda_\mathrm{scale}\cdot\exp\left(M_{t}\right)\cdot \hat{A}_{t,i}.
\end{equation}

For hard prompts, we have $c\in \left(0, 1\right)$ and $\hat{p}\in \left(0, 1\right)$. Based on Equation~\eqref{eq:inequal} and Lemma~\ref{lem:c_range_hatp}, to mitigate biased estimation, $\lambda_\mathrm{scale}$ satisfies:
\begin{equation}
    \lambda_{\mathrm{scale}}^{\mathrm{1}}\in \left( \frac{1 + \frac{(1 - c_{\mathrm{high}}^{\mathrm{hard}}) \hat{p}_t}{1 - \hat{p}_t}}{\exp\left(M_{t}\right)},\ \frac{1 + \frac{(1 - c_{\mathrm{low}}^{\mathrm{hard}}) \hat{p}_t}{1 - \hat{p}_t}}{\exp\left(M_{t}\right)} \right).
\end{equation}
And for incorrect responses in hard prompts, we have:
\begin{equation}
    \begin{aligned}
\hat{A}^{\mathrm{2}}_{t,i}=\frac{\lambda_\mathrm{scale}}{\exp\left(M_{t}\right)}\cdot \hat{A}_{t,i}.
    \end{aligned}
\end{equation}
And we can set:
\begin{equation}
    \lambda_{\mathrm{scale}}^{\mathrm{2}}\in \left( c_{\mathrm{low}}^{\mathrm{hard}}\cdot{\exp\left(M_{t}\right)},\ c_{\mathrm{high}}^{\mathrm{hard}}\cdot{\exp\left(M_{t}\right)} \right).
\end{equation}

For easy prompts, we have $c>1$ and $\hat{p}\in \left(0, 1\right)$, thus for correct answers:
\begin{equation}
\begin{aligned}
    \lambda_{\mathrm{scale}}^3 \in \left( 
    \left(1 + \frac{(1 - c_{\mathrm{high}}^{\mathrm{easy}}) \hat{p}_t}{1 - \hat{p}_t}\right) \cdot \exp(M_t), \right. 
    \left. \left(1 + \frac{(1 - c_{\mathrm{low}}^{\mathrm{easy}}) \hat{p}_t}{1 - \hat{p}_t}\right) \cdot \exp(M_t)
\right).
\end{aligned}
\end{equation}
And for incorrect responses:
\begin{equation}
    \lambda_{\mathrm{scale}}^{\mathrm{4}}\in \left( \frac{c_{\mathrm{low}}^{\mathrm{easy}}}{\exp\left(M_{t}\right)},\ \frac{c_{\mathrm{high}}^{\mathrm{easy}}}{\exp\left(M_{t}\right)} \right).
\end{equation}

In training process with HA-DW, to rectify the biased advantage estimation, there exists a specfic $\lambda_{\mathrm{scale}}$ supposing to satisfy:
\begin{equation}
    \begin{aligned}
        \lambda_{\mathrm{scale}}\in \lambda_{\mathrm{scale}}^{\mathrm{1}}\cup \lambda_{\mathrm{scale}}^{\mathrm{2}}\cup
        \lambda_{\mathrm{scale}}^{\mathrm{3}}\cup \lambda_{\mathrm{scale}}^{\mathrm{4}}
    \end{aligned}
\end{equation}
which denotes:
\begin{equation}\label{eq:lamba-scale-correct}
    \begin{aligned}
        \lambda_{\mathrm{scale}}\in &\left( \frac{1 + \frac{(1 - c_{\mathrm{low}}^{\mathrm{hard}}) \hat{p}_t}{1 - \hat{p}_t}}{\exp\left(M_{t}\right)},\ \frac{1 + \frac{(1 - c_{\mathrm{low}}^{\mathrm{hard}}) \hat{p}_t}{1 - \hat{p}_t}}{\exp\left(M_{t}\right)} \right) \cup\\ 
        &\left( 
        \left(1 + \frac{(1 - c_{\mathrm{high}}^{\mathrm{easy}}) \hat{p}_t}{1 - \hat{p}_t}\right) \cdot \exp(M_t), \right. 
        \left. \left(1 + \frac{(1 - c_{\mathrm{low}}^{\mathrm{easy}}) \hat{p}_t}{1 - \hat{p}_t}\right) \cdot \exp(M_t)
        \right) \cup \\
        &\left( c_{\mathrm{low}}^{\mathrm{hard}}\cdot{\exp\left(M_{t}\right)},\ c_{\mathrm{low}}^{\mathrm{hard}}\cdot{\exp\left(M_{t}\right)} \right) \cup \\
        &\left( \frac{c_{\mathrm{low}}^{\mathrm{easy}}}{\exp\left(M_{t}\right)},\ \frac{c_{\mathrm{high}}^{\mathrm{easy}}}{\exp\left(M_{t}\right)} \right).
    \end{aligned}
\end{equation}

Overall, since the difficulty does not affect the expressions, we can further derive Equation~\eqref{eq:lamba-scale-correct} as follows:
\begin{equation}\label{eq:lambda-scale-simplified}
    \begin{aligned}
         \lambda_{\mathrm{scale}}\in &\left( \frac{1 + \frac{(1 - c_{\mathrm{high}}) \hat{p}_t}{1 - \hat{p}_t}}{\exp\left(D_{t,i}\cdot M_{t}\right)},\ \frac{1 + \frac{(1 - c_{\mathrm{low}}) \hat{p}_t}{1 - \hat{p}_t}}{\exp\left(D_{t,i}\cdot M_{t}\right)} \right)  
        \cup \left( \frac{c_{\mathrm{low}}}{\exp\left(D_{t,i}\cdot M_{t}\right)},\ \frac{c_{\mathrm{high}}}{\exp\left(D_{t,i}\cdot M_{t}\right)} \right) .
    \end{aligned}
\end{equation}

When Equation~\eqref{eq:lambda-scale-simplified} holds, our method HA-DW is efficient in compensating biased advantage estimation.

\subsection{Bias Analysis of Advantage Estimation with Standard-Deviation Normalization}\label{appendix:std_norm}

\begin{figure*}[t]
    \centering
    \includegraphics[width=0.4\linewidth, trim=0 0 0 0,clip]{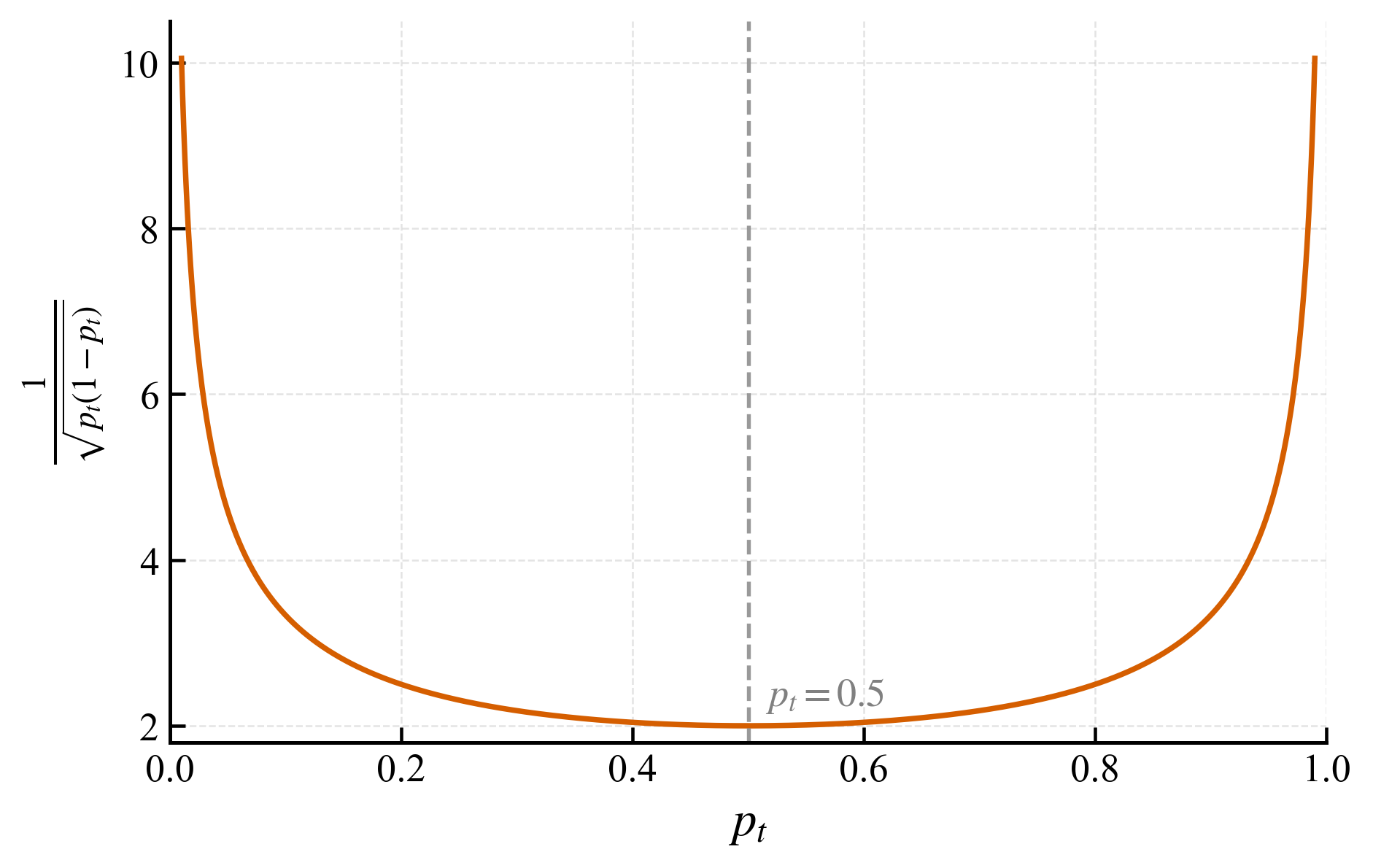}
    \caption{Curve of the normalization term $1/\sqrt{p_t(1-p_t)}$ with respect to $p_t$.}
    \label{fig:std_normalization}
    \vspace{-0.5cm}
\end{figure*}

In this section, we will discuss group-relative advantage estimation with standard-deviation normalization. And the group-relative advantage with normalization is denoted as:

\begin{equation}
\label{eq:group_advantage_std}
\hat{A}_{t,i}
=\frac{r_{t,i} - \hat{p}_t}{\sqrt{\hat{p}_t \cdot (1 - \hat{p}_t)}}, \quad \ \hat{p}_t
=
\frac{1}{G}
\sum_{i=1}^G r_{t,i}.
\end{equation}

Firstly, based on our preceding theoretical analysis, we can observe that when $p_t < 0.5$, the expected reward $p_t$ tends to be overestimated, whereas when $p_t > 0.5$, it tends to be underestimated. Meanwhile, as shown in Figure~\ref{fig:std_normalization}, when $p_t < 0.5$, the normalization term $\frac{1}{\sqrt{p_t \cdot (1-p_t)}}$ decreases monotonically; and when $p_t > 0.5$, the normalization term increases monotonically. Therefore, we can intuitively conclude that the advantage tends to be underestimated when $p_t < 0.5$, and the advantage tends to be overestimated when $p_t > 0.5$ which is consistent with the conclusion obtained above. We provide a detailed analysis of the expectation of the advantage estimator with standard-deviation normalization. Next, we present the main results.

\setcounter{theorem}{3}
\begin{theorem}[Bias under Standard-Deviation Normalization]
\label{thm:std-normalization-bias}
Under the retention condition 
\(\mathcal{S}=\{1\le R\le G-1\}\), for a correct response
\(r_{t,i}=1\), the standard-deviation-normalized group-relative advantage satisfies:
\begin{equation}
    \mathbb{E}\!\left[
    \frac{\hat A_{t,i}}{\sqrt{\hat p_t(1-\hat p_t)}}
    \, \mid  \,
    r_{t,i}=1,\mathcal{S}
\right]
\neq
\frac{A_{t,i}}{\sqrt{p_t(1-p_t)}} .
\end{equation}

Moreover, the bias direction is deterministic in the extreme regimes:
\[
\begin{cases}
\displaystyle
\mathbb{E}\!\left[
    \frac{\hat A_{t,i}}{\sqrt{\hat p_t(1-\hat p_t)}}
    \mid 
    r_{t,i}=1,\mathcal{S}
\right]
<
\frac{A_{t,i}}{\sqrt{p_t(1-p_t)}},
& \text{if } p_t<\frac{1}{G}, \\[2ex]
\displaystyle
\mathbb{E}\!\left[
    \frac{\hat A_{t,i}}{\sqrt{\hat p_t(1-\hat p_t)}}
    \mid 
    r_{t,i}=1,\mathcal{S}
\right]
>
\frac{A_{t,i}}{\sqrt{p_t(1-p_t)}},
& \text{if } p_t>\frac{G-1}{G}.
\end{cases}
\]
Thus, standard-deviation normalization rescales the group-relative advantage
but does not eliminate its intrinsic conditioning bias.
\end{theorem}

\noindent \textit{Proof.} Under the retention condition $\mathcal{S}\coloneqq \{1\le R\le G-1\}$ of the group-relative paradigm, we separately analyze the relationship between the conditional expectation of the advantage estimation $\mathbb{E}\left[\hat{A}_t \mid S, r_{t,i}\right]$ with the expected reward $p_t$ for correct and incorrect responses. 

For correct responses with $r_{t,i}=1$, since we condition on the event that the current response is correct, we necessarily have \(R\ge 1\). Therefore, under \(r_{t,i}=1\), the event \(S=\{1\le R\le G-1\}\) is equivalent to $R\le G-1$. Let \(K\) denote the number of correct responses among the remaining \(G-1\) samples. Then:
\begin{equation}
    K\sim \mathrm{Binomial}(G-1,p_t),
\end{equation}
and the total number of correct responses in the group is $R=1+K$. The condition \(R\le G-1\) implies $1+K\le G-1$,
and hence $K\le G-2$. For a correct response, the normalized group-relative advantage is:
\begin{equation}
    \hat{A}_{t,i}^{+}
=
\frac{1-\hat{p}_t}
{\sqrt{\hat{p}_t(1-\hat{p}_t)}}.
\end{equation}

Leveraging \(\hat{p}_t=R/G=(1+K)/G\), we can obtain:
\begin{equation}
    \hat{A}_{t,i}^{+}
=
\sqrt{\frac{G-1-K}{1+K}}.
\end{equation}

Thus, the conditional expectation can be written as:
\begin{equation}
    \mathbb{E}\left[\hat{A}_{t,i}\mid S, r_{t,i}=1\right]
=
\sum_{k=0}^{G-2}
\mathbb{P}(K=k\mid K\le G-2)
\sqrt{\frac{G-1-k}{1+k}}.
\end{equation}

Since:
\begin{equation}
    \mathbb{P}(K=k\mid K\le G-2)
=
\frac{
\binom{G-1}{k}p_t^k(1-p_t)^{G-1-k}
}{
1-p_t^{G-1}
},
\end{equation}
where
\begin{equation}
    \mathbb{P}(K\le G-2)
=
1-\mathbb{P}(K=G-1)
=
1-p_t^{G-1},
\end{equation}

we obtain:
\begin{equation}
    \mathbb{E}\left[\hat{A}_{t,i}\mid S, r_{t,i}=1\right]
=
\frac{
\sum_{k=0}^{G-2}
\binom{G-1}{k}
p_t^k(1-p_t)^{G-1-k}
\sqrt{\frac{G-1-k}{1+k}}
}{
1-p_t^{G-1}
}
.
\end{equation}

Here, the denominator \(1-p_t^{G-1}\) corresponds to excluding the case where all of the remaining \(G-1\) responses are also correct, i.e., \(R=G\), for which the group standard deviation is zero and the normalized advantage is undefined.

Therefore, the estimation bias for correct responses can be written as:
\begin{equation}
\Delta^{+}(p_t,G)
=
\frac{
\sum_{k=0}^{G-2}
\binom{G-1}{k}
p_t^k(1-p_t)^{G-1-k}
\sqrt{\frac{G-1-k}{1+k}}
}{
1-p_t^{G-1}
}
-
\sqrt{\frac{1-p_t}{p_t}}.
\end{equation}

Equivalently, it can be written as a single fraction:
\begin{equation}
\Delta^{+}(p_t,G)
=
\frac{
\sum_{k=0}^{G-2}
\binom{G-1}{k}
p_t^k(1-p_t)^{G-1-k}
\sqrt{\frac{G-1-k}{1+k}}
-
\left(1-p_t^{G-1}\right)
\sqrt{\frac{1-p_t}{p_t}}
}{
1-p_t^{G-1}
}.
\end{equation}

Since \(p_t\in(0,1)\) and \(1-p_t^{G-1}>0\), the function $\Delta^{+}(p_t,G)$ is continuous on \((0,1)\). When \(p_t\to 0^{+}\), the dominant term in the summation is \(k=0\). Hence,
\begin{equation}
    \mathbb{E}\left[\hat{A}_{t,i}\mid S,r_{t,i}=1\right]
\to
\sqrt{G-1}.
\end{equation}

However, the true advantage satisfies:
\begin{equation}
A_{t,i}^{+}
=
\sqrt{\frac{1-p_t}{p_t}}
\to
+\infty.
\end{equation}
Therefore,
\begin{equation}
\lim_{p_t\to 0^{+}}\Delta^{+}(p_t,G)
=
-\infty.
\end{equation}
This indicates that, for extremely difficult prompts, the advantage of correct responses is severely underestimated. When \(p_t\to 1^{-}\), the valid sampling condition excludes the all-correct case \(R=G\). Conditioned on \(r_{t,i}=1\), the dominant valid case is that exactly one of the remaining \(G-1\) responses is incorrect. Thus,
\begin{equation}
\mathbb{E}\left[\hat{A}_{t,i}\mid S,r_{t,i}=1\right]
\to
\frac{1}{\sqrt{G-1}}.
\end{equation}
Meanwhile,
\begin{equation}
A_{t,i}^{+}
=
\sqrt{\frac{1-p_t}{p_t}}
\to
0.
\end{equation}
Therefore,
\begin{equation}
\lim_{p_t\to 1^{-}}\Delta^{+}(p_t,G)
=
\frac{1}{\sqrt{G-1}}
>
0.
\end{equation}
This implies that, for extremely easy prompts, the advantage of correct responses is overestimated.

Next, we analyze whether the bias:
\begin{equation}
\Delta^{+}(p_t,G)
=
\mathbb{E}\left[\hat{A}_{t,i}\mid S,r_{t,i}=1\right]
-
\sqrt{\frac{1-p_t}{p_t}}
\end{equation}
is positive or negative around \(p_t=0.5\). The key observation is that, under the retention condition $S=\{1\le R\le G-1\}$,
the zero point is generally shifted to the right of \(0.5\). Therefore, while the estimator is provably downward biased for \(p_t<0.5\), it is not necessarily upward biased for all \(p_t>0.5\). Let
\begin{equation}
n=G-1,
\qquad
q=1-p_t,
\end{equation}
and define
\begin{equation}
B_k=\binom{n}{k}p_t^k q^{n-k}.
\end{equation}
Then the expected estimated advantage for correct responses can be rewritten as
\begin{equation}
\mathbb{E}\left[\hat{A}_{t,i}\mid S,r_{t,i}=1\right]
=
\sqrt{\frac{q}{p_t}}
\cdot
\frac{
\sum_{k=0}^{n-1}\sqrt{B_kB_{k+1}}
}{
1-p_t^n
}.
\end{equation}
Therefore,
\begin{equation}
\Delta^{+}(p_t,G)
=
\sqrt{\frac{q}{p_t}}
\left(
\frac{
\sum_{k=0}^{n-1}\sqrt{B_kB_{k+1}}
}{
1-p_t^n
}
-1
\right).
\end{equation}
Since \(\sqrt{q/p_t}>0\), the sign of \(\Delta^{+}(p_t,G)\) is determined by the term inside the parentheses.

By the Cauchy--Schwarz inequality,
\begin{equation}
\left(
\sum_{k=0}^{n-1}\sqrt{B_kB_{k+1}}
\right)^2
\le
\left(
\sum_{k=0}^{n-1}B_k
\right)
\left(
\sum_{k=0}^{n-1}B_{k+1}
\right).
\end{equation}
Moreover,
\begin{equation}
\sum_{k=0}^{n-1}B_k
=
1-p_t^n,
\end{equation}
and
\begin{equation}
\sum_{k=0}^{n-1}B_{k+1}
=
1-q^n.
\end{equation}
Thus,
\begin{equation}
\frac{
\sum_{k=0}^{n-1}\sqrt{B_kB_{k+1}}
}{
1-p_t^n
}
\le
\sqrt{
\frac{1-q^n}{1-p_t^n}
}.
\end{equation}

When \(p_t<0.5\), we have
\begin{equation}
q=1-p_t>p_t,
\end{equation}

Consequently, we can derive that:
\begin{equation}
\Delta^{+}(p_t,G)<0,
\qquad p_t<0.5.
\end{equation}
Equivalently,
\begin{equation}
p_t<0.5
\quad\Longrightarrow\quad
\mathbb{E}\left[\hat{A}_{t,i}\mid S,r_{t,i}=1\right]
<
A_{t,i}^{+}.
\end{equation}
Thus, for difficult prompts with \(p_t<0.5\), the advantage of correct responses is provably underestimated. However, for \(p_t>0.5\), the above argument does not imply
\begin{equation}
\Delta^{+}(p_t,G)>0.
\end{equation}
Instead, due to the retention condition \(S\), the zero point is shifted to the right of \(0.5\). Since
\begin{equation}
\Delta^{+}(0.5,G)<0
\end{equation}
for \(G>2\), while
\begin{equation}
\lim_{p_t\to 1^-}\Delta^{+}(p_t,G)
=
\frac{1}{\sqrt{G-1}}
>0,
\end{equation}
there exists at least one zero point
\begin{equation}
p_G^{+}\in(0.5,1)
\end{equation}
such that
\begin{equation}
\Delta^{+}(p_G^{+},G)=0.
\end{equation}
Therefore, the estimator may remain downward biased even for moderately easy prompts with \(p_t>0.5\), and the bias becomes positive only when \(p_t\) is sufficiently large.

Similarly, for an incorrect response, we can derive
\begin{equation}
    \mathbb{E}\left[\hat{A}_{t,i}\mid S,r_{t,i}=0\right]
=
-
\frac{
\sum_{k=1}^{G-1}
\binom{G-1}{k}
p_t^k(1-p_t)^{G-1-k}
\sqrt{\frac{k}{G-k}}
}{
1-(1-p_t)^{G-1}
}.
\end{equation}
The bias pattern for incorrect responses mirrors that for correct responses. Since the true advantage of an incorrect response is negative, a positive bias indicates that the estimated advantage is less negative than the true one, meaning that the penalty is underestimated. Thus, for easy prompts, incorrect responses tend to be insufficiently penalized, whereas for sufficiently difficult prompts, they may be over-penalized. This behavior is symmetric to the correct-response case, with the corresponding zero point shifted to the opposite side of \(0.5\). Summarizing the above derivation, we can obtain the following conclusion.

Summarizing the above derivation, we show that standard-deviation normalization rescales but does not debias group-relative advantage estimation. The normalized estimator remains affected by the same retention condition as the unnormalized estimator, leading to difficulty-dependent bias for both correct and incorrect responses. In particular, correct responses are underestimated on difficult prompts and overestimated in sufficiently easy regimes, while incorrect responses exhibit a symmetric bias pattern. Therefore, standard-deviation normalization cannot eliminate the intrinsic conditioning bias of group-relative advantage estimation.

\subsection{Non-binary Reward Analysis} \label{appendix:nonbinary} 
In this section, we extend our analysis to \emph{continuous bounded reward distributions} (e.g., Beta and truncated Gaussian scores), which better reflect the behavior of soft verifiers and learned reward models commonly used in practice. The conditional mean estimation error equals
$p_t-\mathbb{E}[\hat{p}_t\mid S_\sigma]$.
It is nonzero whenever conditioning on $S_\sigma$ changes the expected
group mean. Its magnitude and direction depend on the reward
distribution and the update-selection event. Next, we show the main results.

\begin{figure}[th]
  \centering
  \includegraphics[width=0.6\columnwidth]{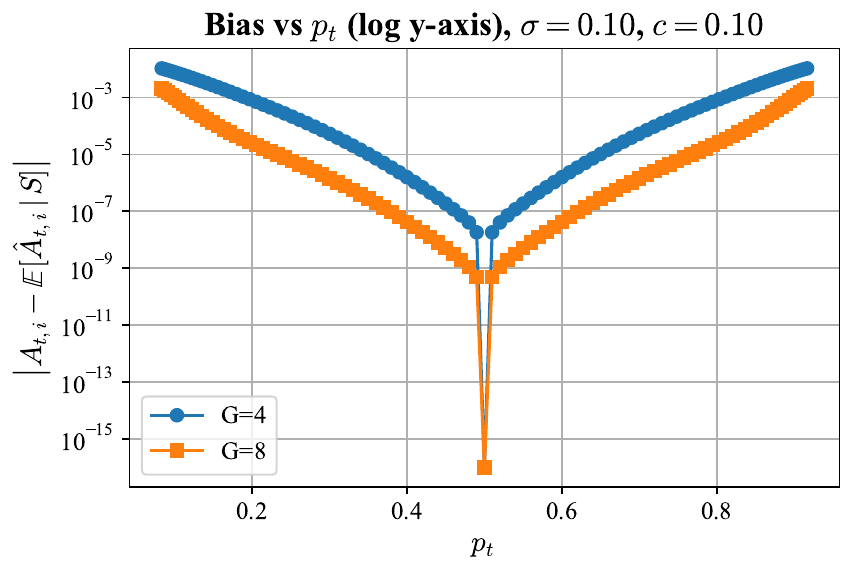}
  \caption{Illustration of advantage bias under truncated Gaussian rewards for different group sizes.}
  \label{appendix:biasdistribution}
\end{figure}

\setcounter{theorem}{5}
\begin{theorem}
    At training step $t$ and let $G \ge 2$, with CDF $F$ and PDF $f$. Given a prompt $x_t\sim\mathcal D$ and draw $G\ge 2$ i.i.d.\ rewards:
\begin{equation}
    r_{t,1},\dots,r_{t,G}\ \overset{\text{i.i.d.}}{\sim}\ \mathcal D(p_t).
\end{equation}
And we extend the binary reward setting to non-binary rewards:
\begin{equation}
    r_{t,i}\in \{0,1\} \rightarrow r_{t,i}\in [0,1].
\end{equation}

The group-relative advantage can be denoted as:
\begin{equation}
    \hat A_{t,i}\coloneqq r_{t,i}-\hat p_t, \quad \hat{p}_t=\frac{1}{G}\sum_{i=1}^G r_{t,i},
\end{equation}
while the expected advantage is defined as:
\begin{equation}
    A_{t,i}\coloneqq r_{t,i}-p_t.
\end{equation}

Fix a constant $\sigma\in[0,1]$ and define the update event:
\begin{equation}
    \begin{aligned}
        S_{\sigma} &:= \bigl\{\exists\, i\neq j:\ |r_{t,i}-r_{t,j}|>\sigma\bigr\} 
\Rightarrow S_{\sigma}^c =\bigl\{\max_i r_{t,i}-\min_i r_{t,i}\le \sigma\bigr\}.
    \end{aligned}
\end{equation}
For $u\in[0,1]$, define $u^+:=\min\{1,u+\sigma\}$, we have:
\begin{equation}
    q(u):=F(u^+)-F(u),
\end{equation}
and:
\begin{equation}
\begin{aligned}
    m(u)&:=\mathbb{E}[r_{t,1}\mid u\le r_{t,1}\le u^+]
=\frac{\int_u^{u^+} x f(x)\,dx}{F(u^+)-F(u)} \quad \text{(when }q(u)>0\text{)}.
\end{aligned}
\end{equation}

Then the probability of a \emph{non-update} is:
\begin{equation}
    \mathbb{P}(S_{\sigma}^c)=G\int_0^1 f(u)\,q(u)^{G-1}\,du,
\end{equation}
and:
\begin{equation}
    \qquad \mathbb{P}(S_{\sigma})=1-\mathbb{P}(S_{\sigma}^c).
\end{equation}

Moreover, we have:
\begin{equation}
    \begin{aligned}
        \mathbb{E}[\hat p_t\mid S_{\sigma}]
=\frac{p_t-\mathbb{E}[\hat p_t\cdot \mathbf{1}_{\{S_{\sigma}^c\}}]}{\mathbb{P}(S_{\sigma})}
    \end{aligned}
\end{equation}
with:
\begin{equation}
    \begin{aligned}
        &\mathbb{E}[\hat p_t\cdot\mathbf{1}_{\{S_{\sigma}^c\}}]
=\int_0^1 \bigl(u+(G-1)m(u)\bigr)\, f(u)\, q(u)^{G-1}\,du.
    \end{aligned}
\end{equation}

Finally, the conditional bias transferred to advantages satisfies, for all $i$, we have:
\begin{equation}
    \begin{aligned}
        \mathbb{E}[\hat A_{t,i}-A_{t,i}\mid S_{\sigma}]=p_t-\mathbb{E}[\hat p_t\mid S_{\sigma}].
    \end{aligned}
\end{equation}
\end{theorem}

\noindent \textit{Proof.} The complement event can be denoted as:
\begin{equation}
    S_{\sigma}^c=\{\max-\min\le \sigma\}
\end{equation}
For absolutely continuous i.i.d.\ samples, the minimum has density:
\begin{equation}
    g_{\min}(u)=G f(u)\bigl(1-F(u)\bigr)^{G-1}.
\end{equation} 
Condition on $\min=u$. The remaining $G-1$ samples are i.i.d.\ with the original law conditioned on $[u,1]$;
imposing $\max\le u^+$ is equivalent to requiring each of those samples lies in $[u,u^+]$.
Thus:
\begin{equation}
    \begin{aligned}
        \mathbb{P}(S_{\sigma}^c\mid \min=u)=\Bigl(\frac{F(u^+)-F(u)}{1-F(u)}\Bigr)^{G-1},
    \end{aligned}
\end{equation}
and multiplying by $g_{\min}(u)$ gives:
\begin{equation}
    \mathbb{P}(S_{\sigma}^c)=G\int_0^1 f(u)\,q(u)^{G-1}du.
\end{equation}

On $S_{\sigma}^c$ and $\min=u$, one sample equals the minimum and the other $G-1$ samples lie in $[u,u^+]$.
By symmetry, the conditional mean of each of the $G-1$ non-minimum samples is $m(u)$, hence:
\begin{equation}
    \begin{aligned}
        \mathbb{E}\Bigl[\sum_{i=1}^G r_{t,i}\ \Big|\ S_{\sigma}^c,\min=u\Bigr]=u+(G-1)m(u).
    \end{aligned}
\end{equation}

So we can derive:
\begin{equation}
    \begin{aligned}
        \mathbb{E}[\hat p_t\cdot\mathbf{1}_{\{S_{\sigma}^c\}}] 
&=\int_0^1 \frac{u+(G-1)m(u)}{G}\, d\mathbb{P}(\min\in du, S_{\sigma}^c) \\
&=\int_0^1 (u+(G-1)m(u)) f(u) q(u)^{G-1}du.
    \end{aligned}
\end{equation}

\begin{corollary}
    For $\mathrm{Beta}(\alpha,\beta)$ reward distribution, the Beta density is:
\begin{equation}
    f(x)=\frac{x^{\alpha-1}(1-x)^{\beta-1}}{B(\alpha,\beta)},
\end{equation}
and the CDF is:
\begin{equation}
    F(x)=I_x(\alpha,\beta)\quad \text{for }x\in[0,1],
\end{equation}
where $B(\cdot,\cdot)$ is the Beta function and $I_x(\alpha,\beta)$ is the regularized incomplete beta function.
In particular:
\begin{equation}
    p_t=\mathbb{E}[r_{t,1}]=\frac{\alpha}{\alpha+\beta}.
\end{equation}
Moreover, we have:
\begin{equation}
\begin{aligned}
    q(u)&=F(u^+)-F(u) 
    =I_{u^+}(\alpha,\beta)-I_u(\alpha,\beta),
\end{aligned}
\end{equation}
and the conditional mean over $[u,u^+]$ admits the closed form:
\begin{equation}
    \begin{aligned}
        m(u)
&=\frac{\int_u^{u^+} x f(x)\,dx}{\int_u^{u^+} f(x)\,dx} 
=\frac{B_{u^+}(\alpha+1,\beta)-B_u(\alpha+1,\beta)}
       {B_{u^+}(\alpha,\beta)-B_u(\alpha,\beta)},
    \end{aligned}
\end{equation}
where $B_x(\cdot,\cdot)$ denotes the (unregularized) incomplete beta function.
\end{corollary}

Consequently, substituting $F,f,q,m$ into conclusions obtained earlier yields explicit one-dimensional
integral expressions (in standard special functions) for $\mathbb{P}(S_{\sigma}^c)$ and $\mathbb{E}[\hat p_t\mid S_{\sigma}]$.

\begin{corollary} \label{cor_}
    Let the reward $Z_{t,1},\dots,Z_{t,G}$ be i.i.d.\ $\mathcal{N}(\mu,\xi^2)$ with $\xi>0$, and define $r_{t,i}$ to be \emph{properly truncated} to $[0,1]$, i.e.\ $r_{t,i}$ has the conditional law:
    \begin{equation}
        \begin{aligned}
            r_{t,i}\ \stackrel{d}{=}\ Z_{t,i}\ \big|\ (0\le Z_{t,i}\le 1),
\qquad i=1,\dots,G.
        \end{aligned}
    \end{equation}

    Let $u^+:=\min\{1,u+c\}$ and define, for $u\in[0,1]$ with $q(u)>0$, we have:
    \begin{equation}
        q(u):=\mathbb{P}\bigl(u\le r_{t,1}\le u^+\bigr),
    \end{equation}
    and:
    \begin{equation}
        m(u):=\mathbb{E}\bigl[r_{t,1}\mid u\le r_{t,1}\le u^+\bigr].
    \end{equation}

    Let $\Phi$ and $\varphi$ be the standard normal CDF and PDF, and set:
    \begin{equation}
        a:=\frac{0-\mu}{\xi},\qquad b:=\frac{1-\mu}{\xi}.
    \end{equation}

    Then the truncated-normal density on $[0,1]$ is:
    \begin{equation}
        \begin{aligned}
            f(x)=\frac{\varphi\!\left(\frac{x-\mu}{\xi}\right)}
          {\sigma\bigl(\Phi(b)-\Phi(a)\bigr)}\,\mathbf{1}_{[0,1]}(x).
        \end{aligned}
    \end{equation}

    Its CDF on $[0,1]$ is:
    \begin{equation}
        F(x)=\frac{\Phi\!\left(\frac{x-\mu}{\xi}\right)-\Phi(a)}{\Phi(b)-\Phi(a)}.
    \end{equation}
    The mean satisfies:
    \begin{equation}
        p_t=\mathbb{E}[r_{t,1}]
=\mu+\xi\,\frac{\varphi(a)-\varphi(b)}{\Phi(b)-\Phi(a)}.
    \end{equation}
    Moreover:
    \begin{equation}
        q(u)=F(u^+)-F(u),
    \end{equation}
    and the conditional mean over $[u,u^+]$ has the standard truncated-normal form:
    \begin{equation}
        m(u)
=\mu+\sigma\,
\frac{\varphi\!\left(\frac{u-\mu}{\xi}\right)-\varphi\!\left(\frac{u^+-\mu}{\xi}\right)}
     {\Phi\!\left(\frac{u^+-\mu}{\xi}\right)-\Phi\!\left(\frac{u-\mu}{\xi}\right)}.
    \end{equation}
Consequently, substituting $F,f,q,m$ to yield explicit one-dimensional
integral expressions for $\mathbb{P}(S_{\sigma}^c)$ and $\mathbb{E}[\hat p_t\mid S_{\sigma}]$ in terms of $\Phi$ and $\varphi$.
\end{corollary}

\begin{tcolorbox}[
    colback=blue!5,
    colframe=blue!60!black,
    boxrule=0.8pt,
    arc=2pt,
    left=6pt,
    right=6pt,
    top=6pt,
    bottom=6pt
]
\paragraph{Remark.}
Theoretical and numerical evaluations under properly truncated Gaussian reward distributions indicate that the group-relative
advantage bias,
$\left|\mathbb{E}[\hat A_{t,i}-A_{t,i}\mid S_{\sigma}]\right|$,
tends to increase as $p_t$ deviates further from $1/2$.
Intuitively, near these extremes, a typical group of samples exhibits reduced dispersion.
As a result, conditioning on the non-degenerate event $S_{\sigma}$ preferentially selects groups with
atypically large variability, which can shift the conditional expected
group reward away from $p_t$ and produce a nonzero conditional mean
estimation error.
\end{tcolorbox}

Figure~\ref{appendix:biasdistribution} illustrates two representative cases corresponding to
group sizes $G=4$ and $G=8$, as predicted by Corollary~\ref{cor_}.
In both settings, the magnitude of the bias $|\mathbb{E}[\hat A_{t,i}-A_{t,i}\mid S_{\sigma}]|$ increases
as $p_t$ moves farther away from $0.5$, corroborating our theoretical analysis.

\section{Supplementary Experiments} \label{appendix:supplementary}

\subsection{Advantage Distribution} \label{appendix:toy-experiment}

Figure~\ref{fig:introduction}(b) compares the same prompts under 8 and 128 rollouts, using the latter as a lower-variance proxy for prompt difficulty. Small-group estimation underweights rare correct responses on hard prompts and under-penalizes rare incorrect responses on easy prompts. And we conducted an assessment of select prompts from the widely used training dataset MATH and DAPO-Math-17k \citep{DBLP:journals/corr/abs-2503-14476} on Qwen3-4B-Base across different rollouts. At first, we evaluated the model's performance on the dataset at rollout=8. From these, we selected four groups of 50 prompts each: groups with single correct or incorrect response. We then evaluated the outcomes of these selected prompts at rollout=128 where enough rollouts can reflect intrinsic difficulty of these prompts. 

\par For those groups with only 1 correct responses at rollout=8, the distribution of the number of correct responses within these groups is shown in Figure~\ref{fig:appendix-c}(a). For the MATH and DAPO-Math-17k datasets, 24 and 15 groups have fewer than 16 correct responses at rollout=128 respectively which suggests that the advantage of correct responses for these prompts are underestimated at rollout=8. And these distinct responses in these most challenging prompts are crucial for pushing the model’s capability frontier, requiring more exploration. Similarly, for prompts with 1 incorrect answer at rollout=8 and we find that 12 and 21 groups have less than 16 incorrect responses with 128 rollouts on MATH and DAPO-Math-17k  which may lead to over-exploitation as Figure~\ref{fig:appendix-c}(b).

\subsection{Ablation Study on $G$} \label{appendix:Group-size}

It is a widely accepted consensus that increasing the number of rollouts effectively mitigates estimation bias \citep{DBLP:journals/corr/abs-2510-04996}. As the group size grows, the empirical advantage distribution converges closer to the true advantage distribution, thereby reducing the variance and bias inherent in the advantage estimation of group-relative RL algorithms. To rigorously validate the effectiveness of our dynamic adjusting approach HA-DW in mitigating estimation bias under constrained sampling conditions, we conducted a comparative analysis of model's training performance across varying rollout sizes. The results presented in Table~\ref{tab:Ablation-on-G} shows that increasing the number of rollouts can, to a certain extent, enhance model performance by providing a more stable baseline. Although scaling up the number of rollouts is a straightforward method to improve performance, its benefits are often capped by computational constraints. As the rollout size increases, the training time grows substantially and out-of-memory (OOM) will occur once it exceeds a certain threshold. Our method offers a more efficient alternative: dynamic advantage adjustment demonstrates superior efficacy even under the limited rollouts condition, and it effectively mitigates the estimation bias that typically plagues low-sample scenarios, achieving robust performance without the need for extensive sampling.

\subsection{Ablation Study on $\lambda_{\text{scale}}$} \label{appendix:lambda-scale}
As illustrated in Section~\ref{appendix:adjusted estimator},there exists a specific scaling factor $\lambda_{\mathrm{scale}}$ satisfying Equation~\eqref{eq:lambda-scale-simplified} to compensate biased advantage estimation. Although the theoretically optimal range of $\lambda_{\mathrm{scale}}$ may vary with prompt difficulty, the prompt-dependent variation is already captured by the difficulty-aware term, while the dynamic difficulty anchor accounts for changes during training. Therefore, $\lambda_{\mathrm{scale}}$ mainly serves as a global calibration factor controlling the overall magnitude of reweighting rather than the relative adjustment across prompts. Table~\ref{tab:differen-lambda} demonstrates the performance of RL training under different values of $\lambda_{\mathrm{scale}}$. When $\lambda_{\mathrm{scale}}=1.3$ or $1.5$, the trained model achieves the best performance across five mathematical benchmarks. The results correspond to our analysis that there exists an optimal value that balances the adjustment across prompts of varying difficulties, thereby enhancing RL training performance. The consistently strong performance around $1.3$--$1.5$ also suggests that precise prompt-level tuning of $\lambda_{\mathrm{scale}}$ is not required in practice.

\subsection{Experiments of Different Domains}
We further extend our experiments to instruction following and tool use. The results, presented in Table~\ref{tab:different_domains}, show that HA-DW delivers robust improvements across diverse RLVR domains.

\begin{table}[h]
  \centering
  \footnotesize
   \renewcommand{\arraystretch}{1.2}
  \setlength{\tabcolsep}{20pt}   
  \begin{tabular}{cccc}
    \hline
    Dataset & 8 & 16 & 8+HA-DW\\
    \hline
    MATH500 & 75.4 & 76.2 & \textbf{78.0} \\
    AIME25 & 19.6 & 19.2 & \textbf{20.4} \\
    AMC23 & 60.3 & 61.6 & \textbf{63.4} \\
    Minerva & 33.8 & 34.2 & \textbf{36.8} \\
    OlympiadBench & 43.5 & 43.9 & \textbf{44.7}\\
    \hline
  \end{tabular}
  
  \caption{Performance of Qwen3-4B-Base trained with: Rollout=8 with GRPO, Rollout=16 with GRPO and Rollout=8 with GRPO+HA-DW. \emph{Rollout=32 with GRPO is out of memory}.}
  \label{tab:Ablation-on-G}
  \vspace{-0.5cm}
\end{table}

\subsection{Experiments on Different Group Sizes}\label{appendix:different_group_size}
We further evaluate the performance on five mathematical benchmarks with group sizes $G=8$ and $G=16$, and further apply HA-DW under $G=8$. The results demonstrate in Table~\ref{tab:Ablation-on-G} show that our method substantially outperforms the gains obtained by merely increasing the group size.

\subsection{Effect on Easy Prompts}
For easy prompts, HA-DW is not primarily intended to further improve the accuracy of already-solved examples. Instead, it corrects distorted update signals that may otherwise over-reinforce correct responses and insufficiently penalize rare incorrect ones, thereby reducing unnecessary exploitation of easy prompts and reallocating learning signals toward more challenging examples. To further examine this effect, we compare the full HA-DW with a variant that applies reweighting only to hard prompts. As shown in Table~\ref{tab:hard_prompt_only}, the full HA-DW achieves a higher average performance than the hard-prompt-only variant, suggesting that correcting the update distortion on easy prompts also contributes to the overall performance gains.

\begin{table}[t]
\centering
\caption{Ablation of HA-DW on easy prompts using Qwen3-4B-Base. The hard-prompt-only variant applies HA-DW reweighting only to hard prompts, while leaving easy-prompt updates unchanged.}
\label{tab:hard_prompt_only}

\resizebox{\linewidth}{!}{
\begin{tabular}{lcccccc}
\toprule
Method 
& MATH500 
& AIME25 
& AMC23 
& Minerva 
& OlympiadBench 
& AVG \\
\midrule

GRPO
& 75.4
& 19.6
& 60.3
& 33.8
& 43.5
& 46.5 \\

GRPO + HA-DW (Hard Prompts Only)
& 77.4
& 20.6
& 62.8
& 35.7
& 44.3
& 48.2 \\

\rowcolor{blue!8}
GRPO + HA-DW
& 78.0
& 20.4
& 63.4
& 36.8
& 44.7
& 48.7 \\

\bottomrule
\end{tabular}
}
\end{table}

\subsection{Ablation Study on Rollout Token Budget}
As HA-DW may induce slightly longer responses, we further examine whether its performance gains can be attributed to additional rollout-generation compute. We train GRPO and GRPO+HA-DW with the same group size, maximum generation length, and a matched total rollout budget of 120M generated response tokens, stopping training once the budget is reached. As shown in Table~\ref{tab:token_budget}, HA-DW consistently outperforms GRPO under the same rollout-token budget, improving the average accuracy from 46.5 to 48.0. This result suggests that the observed gains cannot be explained solely by longer reasoning trajectories or additional rollout-generation compute.

\begin{table}[t]
\centering
\caption{Performance under a matched rollout-token budget. 
GRPO and GRPO+HA-DW use the same group size, maximum generation length, 
and total budget of 120M generated response tokens.}
\label{tab:token_budget}
\resizebox{\linewidth}{!}{
\begin{tabular}{lccccccc}
\toprule
Method 
& Rollout Token Budget 
& MATH500 
& AIME25 
& AMC23 
& Minerva 
& OlympiadBench 
& AVG \\
\midrule
GRPO 
& 120M 
& 75.4 
& 19.6 
& 60.3 
& 33.8 
& 43.5 
& 46.5 \\

\rowcolor{blue!8}
GRPO + HA-DW 
& 120M 
& 77.4 
& 19.8 
& 62.7 
& 35.9 
& 44.3 
& 48.0 \\
\bottomrule
\end{tabular}
}
\end{table}

\subsection{Robustness to Reward Noise}
Since HA-DW increases the influence of hard prompts, we further evaluate its robustness when the reward signals of hard prompts are noisy. Specifically, we randomly select a subset of hard prompts and flip their binary verifier rewards during training, while keeping all evaluation labels and verifier signals unchanged. We compare GRPO and GRPO+HA-DW with and without injected reward noise on Qwen3-4B-Base. As shown in Table~\ref{tab:reward_noise}, reward noise degrades the performance of both methods. Nevertheless, HA-DW retains a modest advantage over GRPO under the noisy setting, achieving an average score of 44.1 compared with 43.5. These results suggest that HA-DW remains effective under moderate reward noise, while its performance is still sensitive to the reliability of reward signals on hard prompts.

\begin{table}[t]
\centering
\caption{Robustness of HA-DW under injected reward noise on hard prompts using Qwen3-4B-Base.}
\label{tab:reward_noise}
\resizebox{\linewidth}{!}{
\begin{tabular}{lcccccc}
\toprule
Method
& MATH500
& AIME25
& AMC23
& Minerva
& OlympiadBench
& AVG \\
\midrule

GRPO
& 75.4
& 19.6
& 60.3
& 33.8
& 43.5
& 46.5 \\

\rowcolor{blue!8}
GRPO + Reward Noise& 72.6
& 16.7
& 56.9
& 30.9
& 40.6
& 43.5 \\

GRPO + HA-DW
& 78.0
& 20.4
& 63.4
& 36.8
& 44.7
& 48.7 \\

\rowcolor{blue!8}
GRPO + HA-DW + Reward Noise
& 73.4
& 16.3
& 58.1
& 31.8
& 40.9
& 44.1 \\

\bottomrule
\end{tabular}
}
\end{table}

\subsection{Significance Analysis}
We conduct a significance analysis on Qwen3-4B-Base across five independent runs and five mathematical benchmarks. The results is demonstrated in Table~\ref{tab:significance_analysis}, showing that HA-DW robustly improves over its corresponding baselines on the evaluated benchmarks. The gains remain stable across random seeds, exhibiting low run-to-run variance, and are statistically significant with all p-values below 0.01. Moreover, HA-DW demonstrates robust effectiveness across different RLVR algorithms, including GRPO, GSPO, and DAPO, indicating that its benefits are reliable, broadly applicable, and not limited to a specific benchmark or algorithmic setting.

\section{Hard Evolving Difficulty Anchor}\label{appendix:hardanchor}
To simplify the update process of evolving belief $C_t$, thereby reducing algorithmic complexity. The synchronization of the model's state can be facilitated through a hard update mechanism, executed at every training step.  
Let $h$ be a hyperparameter denoting the number of most recent training rounds considered.
Let $h$ be hyper-paramter to represent the last $h$ training rounds.
The Equation~\eqref{eq:soft-evolve} can be rewritten as:
\begin{equation}
\begin{aligned}
    C_t^+&=\frac{h-1}{h}C_t^-+\frac{1}{h}b_t
    =\frac{1}{h}\left(\sum_{j=1}^{h-1}b_{t-j}+b_t\right),
\end{aligned}
\end{equation}
which indicates that the belief update is effectuated by directly synthesizing the accuracy information derived from the preceding $h$ batches with observations from the current iteration, and we leave the remaining update procedures intact. 
Although this formulation ignores short-term oscillations in belief updates, it significantly simplifies the overall algorithm.

\section{Prompt}

\begin{tcolorbox}[
    colback=white,          
    colframe=black,         
    colbacktitle=darkgray,  
    coltitle=white,         
    title={Prompt Template}  
]
$\{$question$\}$. Let's think step by step and output the final answer within $\backslash\backslash$boxed$\{\}$.

\end{tcolorbox}

\section{Broader Impact}
This work identifies and mitigates a fundamental bias in group-relative advantage estimation for post-training large language models. By improving the stability and generalization of reasoning-oriented training, HA-DW may support more reliable LLMs for scientific research and complex problem solving. Potential negative impacts are generally aligned with those of advanced LLM reasoning technologies, including misuse and unsafe deployment. We therefore emphasize the importance of fair, safe, and responsible deployment practices.

\section{Case Study}
This appendix demonstrates some output examples generated by policy models trained with GRPO and GRPO+HA-DW. And the results are shown in Figure~\ref{fig:case-study-grpo} and Figure~\ref{fig:case-study-grpo-hadw}.

\begin{table*}[!t]
  \centering  
  \setlength{\tabcolsep}{10pt}
  \begin{tabular}{cccccccc}
    \toprule
     \textbf{$\lambda_{\mathrm{scale}}$} & \textbf{MATH500} & \textbf{AIME25} & \textbf{AMC23} & \textbf{Minerva} & \textbf{OlympiadBench} & \textbf{AVG}\\
    \midrule
    0.5
    &75.4 &18.1 &61.1 &34.2 &43.7 &46.5\\
    0.8
    &76.8 &19.2 &61.3 &34.9 &43.7 &47.2\\
    1.0 
    &76.8 &18.5 &61.6 &36.0 &44.3 &47.4\\
    1.3 
    &\textbf{78.0} &20.4 &\textbf{63.4} &36.8 &\textbf{44.7} &\textbf{48.7}\\
    1.5
    &77.8 &\textbf{20.8} &63.1 &\textbf{37.1} &44.0 &48.6\\
    1.7
    &76.4 &20.0 &\textbf{63.4} &36.4 &44.3 &48.1\\
    2.0
    &76.8 &19.0 &61.9 &35.3 &43.5 &47.3\\
\bottomrule
  \end{tabular}
  \caption{\label{tab:differen-lambda}
    Performance of Qwen3-4B-Base trained with GRPO+HA-DW on different $\lambda_{\mathrm{scale}}$.
  }
\end{table*}

\begin{table*}[t]
\centering
\setlength{\tabcolsep}{20pt}
\begin{tabular}{lcccc}
\toprule
\textbf{Algorithm} & \textbf{IFEval} & \textbf{GPQA} & \textbf{API-Bank} & \textbf{AVG} \\
\midrule

\multicolumn{5}{c}{\textbf{Qwen3-4B-Base}} \\
\midrule
GRPO            & 32.3 & 26.2 & 62.1 & 40.2 \\
GRPO+HA-DW      & 34.0 & 27.8 & 63.6 & 41.8 \\
GSPO            & 32.0 & 26.8 & 61.2 & 40.0 \\
GSPO+HA-DW      & 33.1 & 26.3 & 60.7 & 40.0 \\
DAPO            & 31.1 & 27.2 & 62.6 & 40.3 \\
DAPO+HA-DW      & 33.8 & 29.3 & 64.5 & 42.5 \\
\midrule

\multicolumn{5}{c}{\textbf{Qwen3-8B-Base}} \\
\midrule
GRPO            & 39.4 & 33.8 & 69.6 & 47.6 \\
GRPO+HA-DW      & 41.6 & 32.2 & 72.4 & 48.7 \\
GSPO            & 39.0 & 34.3 & 70.1 & 47.8 \\
GSPO+HA-DW      & 41.9 & 35.4 & 72.0 & 49.8 \\
DAPO            & 40.5 & 33.3 & 71.0 & 48.3 \\
DAPO+HA-DW      & 42.1 & 36.4 & 73.4 & 50.6 \\
\bottomrule
\end{tabular}
\caption{Performance comparison on instruction following and tool-use benchmarks.}
\label{tab:different_domains}
\end{table*}

\begin{table*}[t]
\centering
\setlength{\tabcolsep}{5pt}
\renewcommand{\arraystretch}{1.1}
\resizebox{\textwidth}{!}{%
\begin{tabular}{lcccc}
\toprule
\textbf{Benchmark} & \textbf{GRPO (mean$\pm$std, N=5)} & \textbf{HA-DW (mean$\pm$std, N=5)} & \textbf{Average Improvement} & \textbf{P-Value} \\
\midrule
MATH500       & 75.44 $\pm$ 0.30 & 77.88 $\pm$ 0.23 & +2.44 & $3.08 \times 10^{-5}$ \\
AIME25        & 19.04 $\pm$ 0.61 & 20.44 $\pm$ 0.30 & +1.40 & $7.59 \times 10^{-3}$ \\
AMC23         & 60.42 $\pm$ 0.34 & 63.58 $\pm$ 0.41 & +3.16 & $6.71 \times 10^{-6}$ \\
Minerva       & 33.68 $\pm$ 0.41 & 36.78 $\pm$ 0.25 & +3.10 & $4.33 \times 10^{-4}$ \\
OlympiadBench & 43.42 $\pm$ 0.44 & 44.88 $\pm$ 0.20 & +1.46 & $1.20 \times 10^{-3}$ \\
\bottomrule
\end{tabular}%
}
\caption{Significance analysis of HA-DW against GRPO on Qwen3-4B-Base over five independent runs.}
\label{tab:significance_analysis}
\end{table*}

\begin{figure*}[!ht]
    \centering
    \includegraphics[width=\linewidth, trim=180 160 220 180,clip]{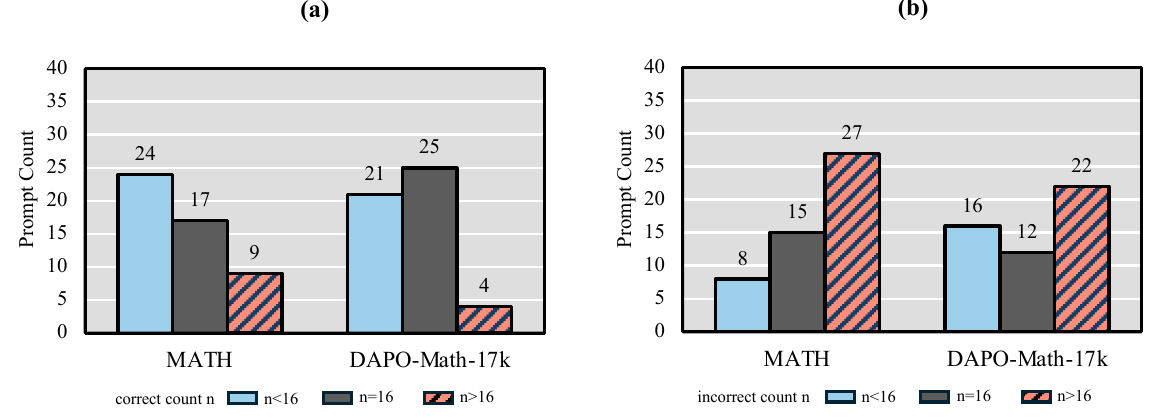}
    \caption{The distribution of prompts by the number of correct and incorrect responses on the MATH dataset and DAPO-Math-17k under 8 and 128 rollouts of Qwen3-4B-Base.}
    \label{fig:appendix-c}
\end{figure*}

\begin{table*}[t]
\centering
\small
\setlength{\tabcolsep}{3pt}

\renewcommand{\arraystretch}{1.25}
\setlength{\extrarowheight}{1pt}

\resizebox{\textwidth}{!}{%
\begin{tabular}{clcccccc}
\toprule
\textbf{} & \textbf{Hyperparameter} & \textbf{GRPO} & \textbf{GRPO+HA-DW} & \textbf{GSPO} & \textbf{GSPO+HA-DW} & \textbf{DAPO} & \textbf{DAPO+HA-DW} \\
\midrule

\multirow[c]{6}{*}{\textbf{General}}
& nnode                 & 1     & 1     & 1     & 1     & 1     & 1 \\
& gpus per node         & 8     & 8     & 8     & 8     & 8     & 8 \\
& use kl in reward      & False & False & False & False & False & False \\
& use kl loss           & False & False & False & False & False & False \\
& tensor parallel size  & 1     & 1     & 1     & 1     & 1     & 1 \\
& test frequency        & 5     & 5     & 5     & 5     & 5     & 5 \\
\midrule

\multirow[c]{9}{*}{\textbf{Training}}
& train batch size & 256 & 256 & 256 & 256 & 256 & 256 \\
& mini batch size  & 16  & 16  & 16  & 16  & 16  & 16 \\
& micro batch size & 4   & 4   & 4   & 4   & 4   & 4 \\
& epoch            & 3   & 3   & 3   & 3   & 9   & 9 \\
& gradient clip    & 1.0 & 1.0 & 1.0 & 1.0 & 1.0 & 1.0 \\
& optimizer        & AdamW & AdamW & AdamW & AdamW & AdamW & AdamW \\
& warmup steps     & 10  & 10  & 10  & 10  & 10  & 10 \\
& weight decay     & 0.1 & 0.1 & 0.1 & 0.1 & 0.1 & 0.1 \\
& learning rate    & $1\times10^{-6}$ & $1\times10^{-6}$ & $1\times10^{-6}$ & $1\times10^{-6}$ & $1\times10^{-6}$ & $1\times10^{-6}$ \\
\midrule

\multirow[c]{2}{*}{\textbf{Clipping}}
& clip-high & 0.2  & 0.2  & 0.0004  & 0.0004  & 0.28  & 0.28 \\
& clip-low  & 0.2  & 0.2  & 0.0003  & 0.0003  & 0.2   & 0.2 \\
\midrule

\multirow[c]{6}{*}{\textbf{Rollout}}
& max prompt length   & 1024 & 1024 & 1024 & 1024 & 1024 & 1024 \\
& max response length & 4096 & 4096 & 4096 & 4096 & 4096 & 4096 \\
& rollout.n           & 8    & 8    & 8    & 8    & 8    & 8 \\
& do sample           & False& False& False& False& False& False \\
& filtering           & False& False& False& False& False& False \\
& dynamic batch size  & True & True & True & True & True & True \\
\bottomrule
\end{tabular}
}

\caption{Hyperparameter settings for group-relative methods.}
\label{table:hyperparameter}
\end{table*}

\clearpage

\begin{figure*}[t]
\centering
\begin{tcolorbox}[
    colback=white,          
    colframe=black,         
    colbacktitle=darkgray,  
    coltitle=white,         
    title={Case Study}  

]

\textbf{Question:}\\[0.3em]
\small
\linespread{1.0}\selectfont
In the $xy$-plane, a circle of radius $4$ with center on the positive $x$-axis is tangent to the $y$-axis at the origin, and a circle with radius $10$ with center on the positive $y$-axis is tangent to the $x$-axis at the origin. What is the slope of the line passing through the two points at which these circles intersect? The final answer can be written in the form $\frac{m}{n}$, where $m$ and $n$ are relatively prime positive integers. What is $m+n$? Let's think step by step and output the final answer within $\backslash\backslash$ boxed\{\}.\\

{\large \textbf{Answer}: \small 7}

\tcblower 

\textbf{GRPO Solution:}\\[0.3em]
\small
\linespread{1.3}\selectfont
To solve the problem, we need to find the points of intersection of the two circles and then determine the slope of the line passing through these points.\\

1. Equation of the Circles:
The first circle has a radius of 4 and is centered on the positive $x$-axis. Since it is tangent to the $y$-axis at the origin, its center is at $(4, 0)$. The equation of this circle is:
$ (x - 4)^2 + y^2 = 16 $. The second circle has a radius of 10 and is centered on the positive $y$-axis. Since it is tangent to the $x$-axis at the origin, its center is at $(0, 10)$. The equation of this circle is:
$ x^2 + (y - 10)^2 = 100 $.\\

2. Finding the Points of Intersection:
We need to solve the system of equations:
\[
\begin{cases}
    (x - 4)^2 + y^2 = 16\\
    x^2 + (y - 10)^2 = 100
\end{cases}
\]

Expanding both equations, we get:
\[
\begin{cases}
    x^2 - 8x + 16 + y^2 = 16 \implies x^2 - 8x + y^2 = 0 \\
    x^2 + y^2 - 20y + 100 = 100 \implies x^2 + y^2 - 20y = 0
\end{cases}
\]

Subtract from the second equation:
$ (x^2 + y^2 - 20y) - (x^2 - 8x + y^2) = 0 \implies -20y + 8x = 0 \implies 8x = 20y \implies x = \frac{5y}{2} $. \\

Substitute $x = \frac{5y}{2}$ into the first equation :
$ \left(\frac{5y}{2}\right)^2 - 8\left(\frac{5y}{2}\right) + y^2 = 0 \implies \frac{25y^2}{4} - 20y + y^2 = 0 \implies \frac{25y^2 + 4y^2}{4} - 20y = 0 \implies \frac{29y^2}{4} - 20y = 0 $
$ 29y^2 - 80y = 0 \implies y(29y - 80) = 0 $
So, $y = 0$ or $y = \frac{80}{29}$. If $y = 0$, then $x = \frac{5 \cdot 0}{2} = 0$, which is the origin and not a point of intersection of the two circles. If $y = \frac{80}{29}$, then $x = \frac{5 \cdot \frac{80}{29}}{2} = \frac{200}{29}$.\\

So, the points of intersection are $\left(\frac{200}{29}, \frac{80}{29}\right)$ and $(0, 0)$, but we already know $(0, 0)$ is not a valid intersection point for the two circles. Therefore, the only valid intersection point is $\left(\frac{200}{29}, \frac{80}{29}\right)$.\\

3. Finding the Slope:
The other point of intersection can be found by symmetry or by solving the equations again, but we know the circles are symmetric about the line $y = x$. So, the other point of intersection is $\left(\frac{80}{29}, \frac{200}{29}\right)$. The slope of the line passing through the points $\left(\frac{200}{29}, \frac{80}{29}\right)$ and $\left(\frac{80}{29}, \frac{200}{29}\right)$ is:
$ \text{slope} = \frac{\frac{200}{29} - \frac{80}{29}}{\frac{80}{29} - \frac{200}{29}} = \frac{\frac{120}{29}}{-\frac{120}{29}} = -1 $. So, the slope is $-1$, and in the form $\frac{m}{n}$, we have $m = -1$ and $n = 1$. \\

Therefore, $m + n = -1 + 1 = 0$.\\

The final answer is:
$ \boxed{-1} $

\end{tcolorbox}
\caption{An example of GRPO} 
\label{fig:case-study-grpo}
\end{figure*}

\begin{figure*}[t]
\centering
\begin{tcolorbox}[
    colback=white,          
    colframe=black,         
    colbacktitle=darkgray,  
    coltitle=white,         
    title={Case Study}  

]

\textbf{GRPO+HA-DW Solution:}\\[0.3em]
\small
\linespread{1.0}\selectfont
To solve the problem, we need to find the points of intersection of the two circles and then determine the slope of the line passing through these points.\\

1. Equation of the first circle:\\
The first circle has a radius of 4 and is centered on the positive $x$-axis. Since it is tangent to the $y$-axis at the origin, the center of the circle is at $(4, 0)$. The equation of the first circle is:
$(x - 4)^2 + y^2 = 16 $.\\

2. Equation of the second circle:\\
The second circle has a radius of 10 and is centered on the positive $y$-axis. Since it is tangent to the $x$-axis at the origin, the center of the circle is at $(0, 10)$. The equation of the second circle is:
$ x^2 + (y - 10)^2 = 100 $.\\

3. Find the points of intersection:
We need to solve the system of equations:
\[
  \begin{cases}
  (x - 4)^2 + y^2 = 16 \\
  x^2 + (y - 10)^2 = 100 
  \end{cases}
  \]

First, expand both equations:
\[
\begin{cases}
    x^2 - 8x + 16 + y^2 = 16 \implies x^2 - 8x + y^2 = 0\\
    x^2 + y^2 - 20y + 100 = 100 \implies x^2 + y^2 - 20y = 0
\end{cases}
\]

Subtract the first expanded equation from the second:
 $(x^2 + y^2 - 20y) - (x^2 - 8x + y^2) = 0 - 0 $
$ -20y + 8x = 0 \implies 8x = 20y \implies x = \frac{5}{2}y $.\\

Substitute $x = \frac{5}{2}y$ into the first expanded equation:
$ \left(\frac{5}{2}y\right)^2 - 8\left(\frac{5}{2}y\right) + y^2 = 0 $. This gives us two solutions for $y$:
$ y = 0 \quad \text{or} \quad \frac{29}{4}y = 20 \implies y = \frac{80}{29} $.\\

For $y = 0$: $ x = \frac{5}{2}(0) = 0 $. So, one intersection point is $(0, 0)$, which is the origin. For $y = \frac{80}{29}$: $ x = \frac{5}{2}\left(\frac{80}{29}\right) = \frac{200}{29} $. So, the other intersection point is $\left(\frac{200}{29}, \frac{80}{29}\right)$.\\

4. Find the slope of the line passing through the points $(0, 0)$ and $\left(\frac{200}{29}, \frac{80}{29}\right)$:
The slope $m$ of the line passing through two points $(x_1, y_1)$ and $(x_2, y_2)$ is given by:
$ m = \frac{y_2 - y_1}{x_2 - x_1} $
Substituting the points $(0, 0)$ and $\left(\frac{200}{29}, \frac{80}{29}\right)$:
$ m = \frac{\frac{80}{29} - 0}{\frac{200}{29} - 0} = \frac{\frac{80}{29}}{\frac{200}{29}} = \frac{80}{200} = \frac{2}{5} $.\\\

Therefore, the slope of the line is $\frac{2}{5}$, and $m + n = 2 + 5 = 7$.\\

The final answer is \boxed{7}.

\end{tcolorbox}
\caption{An example of GRPO+HA-DW} 
\label{fig:case-study-grpo-hadw}
\end{figure*}


\clearpage
\section*{NeurIPS Paper Checklist}



\begin{enumerate}

\item {\bf Claims}
    \item[] Question: Do the main claims made in the abstract and introduction accurately reflect the paper's contributions and scope?
    \item[] Answer: \answerYes{} 
    \item[] Justification: The main claims in the abstract and introduction are aligned with the paper's theoretical and empirical contributions. Specifically, we clearly state the scope of our analysis, the assumptions under which the bias characterization holds, and the experimental settings used to validate the proposed method.

\item {\bf Limitations}
    \item[] Question: Does the paper discuss the limitations of the work performed by the authors?
    \item[] Answer: \answerYes{} 
    \item[] Justification: We discuss the limitations in the dedicated Limitations section. We clarify that HA-DW targets biased advantage estimation in group-relative RL and is therefore mainly scoped to group-wise optimization methods. We also note that, although HA-DW may apply beyond standard RLVR, our analysis and experiments focus on RLVR reasoning tasks, leaving broader domains for future work.

\item {\bf Theory assumptions and proofs}
    \item[] Question: For each theoretical result, does the paper provide the full set of assumptions and a complete (and correct) proof?
    \item[] Answer: \answerYes{} 
    \item[] Justification: We state the assumptions for each theoretical result and provide complete derivations and proofs in the main paper and appendix. The theoretical results characterize the finite-sample bias of group-relative advantage estimation under the specified sampling setting, and the corresponding formulas, propositions, and proofs are numbered and cross-referenced.

    \item {\bf Experimental result reproducibility}
    \item[] Question: Does the paper fully disclose all the information needed to reproduce the main experimental results of the paper to the extent that it affects the main claims and/or conclusions of the paper (regardless of whether the code and data are provided or not)?
    \item[] Answer: \answerYes{} 
    \item[] Justification: We provide the experimental settings needed to reproduce the main results, including the evaluated benchmarks, model backbones, group-relative training methods, reward definitions, sampling configurations, optimization hyperparameters, and evaluation protocols. We also report results across multiple runs where applicable, together with implementation details in the main paper and appendix.

\item {\bf Open access to data and code}
    \item[] Question: Does the paper provide open access to the data and code, with sufficient instructions to faithfully reproduce the main experimental results, as described in supplemental material?
    \item[] Answer: \answerYes{} 
    \item[] Justification: We provide an anonymized repository containing the code, data preparation instructions, environment setup, and commands needed to reproduce the main experimental results. The repository is anonymized for the submission stage and includes sufficient documentation to support faithful reproduction.

\item {\bf Experimental setting/details}
    \item[] Question: Does the paper specify all the training and test details (e.g., data splits, hyperparameters, how they were chosen, type of optimizer) necessary to understand the results?
    \item[] Answer: \answerYes{} 
    \item[] Justification: The paper specifies the experimental settings necessary to understand the results, including the training and evaluation datasets, model backbones, data splits or evaluation protocols, optimization settings, hyperparameters, reward definitions, and group-wise sampling configurations. Additional training and implementation details are provided in the appendix.

\item {\bf Experiment statistical significance}
    \item[] Question: Does the paper report error bars suitably and correctly defined or other appropriate information about the statistical significance of the experiments?
    \item[] Answer: \answerYes{} 
    \item[] Justification: We report statistical significance information for the main experiments, including results over multiple runs with mean and standard deviation where applicable, as well as significance tests and corresponding P-values. These analyses are used to assess whether the observed improvements of HA-DW over group-relative baselines are statistically reliable.

\item {\bf Experiments compute resources}
    \item[] Question: For each experiment, does the paper provide sufficient information on the computer resources (type of compute workers, memory, time of execution) needed to reproduce the experiments?
    \item[] Answer: \answerYes{} 
    \item[] Justification: The paper reports the GPU resources used for the experiments, but does not provide complete information about memory usage, execution time, or total compute required for reproducing each experiment.
    
\item {\bf Code of ethics}
    \item[] Question: Does the research conducted in the paper conform, in every respect, with the NeurIPS Code of Ethics \url{https://neurips.cc/public/EthicsGuidelines}?
    \item[] Answer: \answerYes{} 
    \item[] Justification: We have reviewed the NeurIPS Code of Ethics and believe that the research conducted in this paper conforms to it. The work does not involve human subjects, private or sensitive personal data, or deceptive data collection, and we have taken care to preserve anonymity in the submission.

\item {\bf Broader impacts}
    \item[] Question: Does the paper discuss both potential positive societal impacts and negative societal impacts of the work performed?
    \item[] Answer: \answerYes{} 
    \item[] Justification: The Broader Impacts section discusses potential positive impacts, such as supporting more reliable LLMs for scientific research and complex problem solving, as well as potential negative impacts.
    
\item {\bf Safeguards}
    \item[] Question: Does the paper describe safeguards that have been put in place for responsible release of data or models that have a high risk for misuse (e.g., pre-trained language models, image generators, or scraped datasets)?
    \item[] Answer: \answerNA{} 
    \item[] Justification: This paper does not release any high-risk models, pretrained language models, image generators, or scraped datasets. The work focuses on theoretical analysis and methodological improvements for group-relative reinforcement learning, so no additional safeguards for responsible release are applicable.

\item {\bf Licenses for existing assets}
    \item[] Question: Are the creators or original owners of assets (e.g., code, data, models), used in the paper, properly credited and are the license and terms of use explicitly mentioned and properly respected?
    \item[] Answer: \answerYes{} 
    \item[] Justification: We use publicly available open-source assets, including benchmark datasets, model backbones, and relevant codebases, and properly cite their original sources in the paper. We respect the corresponding licenses and terms of use for these assets.

\item {\bf New assets}
    \item[] Question: Are new assets introduced in the paper well documented and is the documentation provided alongside the assets?
    \item[] Answer: \answerNA{} 
    \item[] Justification: The paper does not introduce or release new assets such as datasets, model checkpoints, or standalone code packages. Therefore, asset documentation requirements are not applicable to this work.

\item {\bf Crowdsourcing and research with human subjects}
    \item[] Question: For crowdsourcing experiments and research with human subjects, does the paper include the full text of instructions given to participants and screenshots, if applicable, as well as details about compensation (if any)? 
    \item[] Answer: \answerNA{} 
    \item[] Justification: This work does not involve crowdsourcing experiments, human-subject studies, or participant-based data collection. Therefore, participant instructions, screenshots, and compensation details are not applicable.

\item {\bf Institutional review board (IRB) approvals or equivalent for research with human subjects}
    \item[] Question: Does the paper describe potential risks incurred by study participants, whether such risks were disclosed to the subjects, and whether Institutional Review Board (IRB) approvals (or an equivalent approval/review based on the requirements of your country or institution) were obtained?
    \item[] Answer: \answerNA{} 
    \item[] Justification: This work does not involve crowdsourcing experiments or research with human subjects. Therefore, IRB approval, participant risk disclosure, and related review procedures are not applicable.

\item {\bf Declaration of LLM usage}
    \item[] Question: Does the paper describe the usage of LLMs if it is an important, original, or non-standard component of the core methods in this research? Note that if the LLM is used only for writing, editing, or formatting purposes and does \emph{not} impact the core methodology, scientific rigor, or originality of the research, declaration is not required.
    \item[] Answer: \answerNA{} 
    \item[] Justification: LLMs were used only for writing assistance, editing, and language polishing, and did not contribute to the core methodology, scientific claims, experiments, or theoretical results of this research. Therefore, declaration of LLM usage is not applicable under the checklist guideline.

\end{enumerate}

\end{document}